\documentclass{article} %
\usepackage{iclr2027_conference,times}

\usepackage{amsmath,amsfonts,bm}

\def\eqref#1{equation~\ref{#1}}

\def\1{\bm{1}}

\DeclareMathAlphabet{\mathsfit}{\encodingdefault}{\sfdefault}{m}{sl}
\SetMathAlphabet{\mathsfit}{bold}{\encodingdefault}{\sfdefault}{bx}{n}

\usepackage{subcaption}
\usepackage{paralist}
\usepackage{hyperref}
\usepackage{microtype}
\usepackage{url}
\usepackage{booktabs}
\usepackage{tabularx}
\usepackage{graphicx}
\hypersetup{
  colorlinks=true,
  citecolor=blue
}
\usepackage{wrapfig}
\usepackage{lipsum}

\usepackage{amssymb}
\usepackage{algorithm}
\usepackage{algpseudocode}
\usepackage{amsthm}
\usepackage{caption}
\usepackage{enumitem}
\usepackage{subcaption}
\usepackage{nicefrac}

\theoremstyle{plain}
\newtheorem{proposition}{Proposition}[section]
\newtheorem{corollary}[proposition]{Corollary}

\newcommand{\compactop}[3]{%
  #1_{\raisebox{0.15ex}{$\scriptscriptstyle #2$}}%
    ^{\raisebox{-0.15ex}{$\scriptscriptstyle #3$}}%
}
\renewcommand{\paragraph}[1]{\textbf{#1}\ }

\title{Hierarchical Cooperative GFlowNets}

\author{%
Pedro Dall'Antonia$^{1}$\quad
Juan Belieni$^{1}$\quad
Ana Carolina Erthal$^{1}$\quad
Saullo H.\,G.\ de Oliveira$^{1}$\\
\textbf{Eliezer de Souza da Silva}$^{2}$\quad
\textbf{Diego Mesquita}$^{1}$\\[4pt]
$^{1}$School of Applied Mathematics (EMAp), Getulio Vargas Foundation, Rio de Janeiro, Brazil\\
$^{2}$University of Coimbra, CISUC/LASI, DEI}

\newcommand{\ind}[1]{\mathbf{1}\!\left\{#1\right\}}

\iclrfinalcopy
\begin{document}

\maketitle
\lhead{Preprint}

\begin{abstract}
Generative Flow Networks (GFlowNets) are trained to sample objects from compositional spaces with probability proportional to a reward function, but ineffective exploration can cause learned samplers to concentrate on a few high-reward regions. 
An intuitive remedy is to train multiple GFlowNets
independently in parallel and combine their samples. However, this does not prevent multiple models from collapsing onto the same high-reward areas, leaving large regions of the state space unexplored. 
With this in mind, we propose Hierarchical Cooperative GFlowNets (HCG), an approach that dynamically partitions the search space to coordinate discovery without relying on hand-crafted boundaries.
HCG imposes an ordering in which the highest-ranked model may explore the entire space, while lower-ranked models are progressively pushed toward regions not covered by higher-ranked ones. 
We provide a theoretical characterization of how specific learning objectives can shift probability mass away from designated areas, instantiating this mechanism within our architecture.
A shared classifier coordinates the hierarchy by partitioning the search space, keeping the GFlowNets independently parameterized and enabling asynchronous training.
Furthermore, we derive a normalizer-weighted aggregation rule for the ensemble and prove that it recovers the target reward-proportional distribution under the stated assumptions.
Finally, experiments across multimodal distributions, sequence generation, combinatorial optimization, and biological sequence design show that HCG improves mode coverage and target approximation, translating increases in the number of coordinated samplers into broader, complementary exploration.
    \looseness=-5

\end{abstract}

\section{Introduction}
\label{sec:introduction}

Generative Flow Networks (GFlowNets) are a class of generative models trained to learn stochastic policies that sample structured objects with probability proportional to a non-negative reward function \citep{bengio2021flow,bengio2023foundations}. By matching a reward-based distribution %
GFlowNets provide a natural mechanism for generating diverse collections of high-reward objects, with demonstrated applications in challenging combinatorial problems in various scientific domains such as biological sequence design and molecular discovery \citep{jain2022biological,zhang2023flows}, causal discovery \citep{deleu2022bayesian,silvaadele2026}, as well as language-model alignment and post-training \citep{kwon2024gdpo,bartoldson2025, fawkes2026ftrajectory}. 

Naturally, as search spaces grow, effective exploration requires increasingly larger computational budgets, making training multiple parallel GFlowNets an intuitive avenue to translate additional compute into broader coverage. However, naively training multiple GFlowNets targeting the same reward may allocate probability to identical regions, limiting broad complementary coverage. Existing approaches to distribute GFlowNet training exploit useful decompositions when they are available. For instance, \citet{dasilva2024embarrassingly} leverage the log-additive structure of Bayesian posteriors to distribute independent reward factors across multiple GFlowNets, while \citet{silva2025generalization} rely on a user-defined decomposition of the state graph to define subproblems for the individual GFlowNets. \looseness=-1

Exploration can also be improved by pairing a main GFlowNet with an auxiliary exploratory one, whose trajectories are used to improve the main sampler's coverage. Sibling-augmented GFlowNets~\citep{madan2025sibling} add a novelty bonus to the explorer's reward, pushing it toward rarely visited states, while Adaptive Complementary Exploration~\citep{dallantonia2026avoid} leverages a balance condition that steers the auxiliary sampler toward high-reward regions the main GFlowNet undersamples. These methods, however, do not provide mechanisms to utilize many parallel samplers. \looseness=-1

The idea of using multiple GFlowNets to broaden exploration was introduced by Boosted GFlowNets~\citep{dallantonia2026boostedgflownetsimprovingexploration}, which build an ensemble in stages: each new GFlowNet is trained on the reward mass not yet captured by the frozen members trained before it, and the ensemble is sampled by selecting each member with probability proportional to its normalizing constant. This construction, however, is inherently sequential, since a member can only be trained after all of its predecessors have been trained and frozen, so the training time grows with the number of members. Our goal is instead to train all components simultaneously. \looseness=-1

We argue that coordinating many samplers requires imposing structure among them. To maximize space coverage, our approach---\emph{Hierarchical Cooperative GFlowNets} (HCG)---organizes the GFlowNets into a hierarchy in which regions covered by higher-ranked models become unavailable to lower-ranked ones, progressively pushing the latter toward unexplored areas. For this coordination to scale, coverage information is centralized in a shared coordinator rather than exchanged pairwise, whose cost would grow quadratically with the ensemble size. The coordinator consists of a classifier that, in essence, implements a weighted ratio estimator for the probabilities the parallel GFlowNets assign to each terminal state. 

We bake this classifier into the individual GFlowNets' losses to mask regions well covered by higher-hierarchy GFlowNets while keeping the target reward unchanged elsewhere. The classifier is trained with samples weighted by two factors: each model's normalizing constant, so that its assignments reflect coverage rather than sampling concentration; and monotonic priority weights that respect the hierarchy, so that overlaps are resolved in favor of the higher-ranked model.

Our theoretical analysis provides a provably correct recipe to combine HCG's individual models into a linear mixture of GFlowNets, as well as a characterization of each model's behavior both inside and outside their assigned regions.
Empirically, we show that HCG provides better mode discovery than methods that explicitly incentivize exploration, particularly in settings where high-reward modes are sparse. Furthermore, we find that increasing the number of components leads to the discovery of additional modes and faster convergence toward the target distribution.

\noindent\textbf{Contributions.} In summary, our main contributions are:

\begin{enumerate}[itemsep=0pt,topsep=0pt,left=8pt]
    \item We introduce hierarchical ownership as a dynamic mechanism for coordinating exploratory coverage among multiple GFlowNets without a predefined partition of the search space;
    \item We characterize a learning objective that avoids already covered regions while preserving the target elsewhere and instantiate it within TB. 
    We further derive a normalizer-weighted aggregation rule that recovers the original reward-proportional target;
    \item We develop an amortized ownership mechanism that removes direct cross-model dependencies from local GFlowNet updates and supports distributed training;
    \item We empirically evaluate HCG across synthetic, combinatorial, and biological sequence design tasks, showing improved exploration and mode discovery over existing exploration methods, and that increasing the number of components leads to broader coverage.
\end{enumerate}

\section{Background}
\label{sec:background}

\paragraph{Notation.}
Let $G = (\mathcal{S}, A)$ be a directed acyclic graph (DAG), where $\mathcal{S}$ is the set of vertices (states) and $A \subseteq \mathcal{S} \times \mathcal{S}$ is the set of edges (transitions, or actions) between pairs of vertices.
Let $s_0 \in \mathcal{S}$ represent the initial state, which has no incoming edges, and $\mathcal{X} \subseteq \mathcal{S}$ be the set of terminal states.
A trajectory $\tau$ is any sequence $(s_0, \cdots, s_T)$ with $s_T \in \mathcal{X}$ and $(s_t,s_{t+1}) \in A$ for all $t$.
Notice that $s_T$ is also usually denoted by $x$ when it represents a terminal state.

A forward policy $P_F : \mathcal{S} \times \mathcal{S} \to [0, 1]$ defines a probability distribution over the children of each state $s \in \mathcal{S}$.
Similarly, a backward policy $P_B : \mathcal{S} \times \mathcal{S} \to [0, 1]$ defines a probability distribution over the parents of each state, equivalently defining a policy on the transposed graph obtained by reversing the edges of $G$.
The forward and backward policies induce probabilities over complete trajectories,
\[
P_F(\tau) = \compactop{\prod}{t=0}{T-1}P_F(s_{t+1}\mid s_t),
\qquad
P_B(\tau\mid x) = \compactop{\prod}{t=0}{T-1}P_B(s_t\mid s_{t+1}).
\]

\paragraph{Generative Flow Networks.}
Given a positive reward function $R:\mathcal X\rightarrow\mathbb R_{>0}$, the goal of a GFlowNet is to learn an amortized trajectory sampler over $G$ such that, for every $x \in \mathcal{X}$,
\begin{equation}
    P_\top(x) := { \sum_{\tau:\tau\to x}P_F(\tau)
    \propto R(x)}.
\end{equation}
Directly enforcing this terminal marginal requires summing over all trajectories terminating at $x$ and is generally intractable. Instead, GFlowNets are trained by enforcing balance conditions whose satisfaction implies the desired reward-proportional terminal distribution. In this work, we focus on the trajectory balance (TB) condition~\citep{malkin2022trajectory}, which introduces a backward policy $P_B$ and a normalizing constant $Z$ and requires, for every complete trajectory $\tau$ terminating at $x$,
\begin{equation}
    ZP_F(\tau)=R(x)P_B(\tau\mid x).
    \label{eq:tb-condition}
\end{equation}
Summing this identity over all trajectories terminating at $x$, and using $\sum_{\tau:\tau\to x}P_B(\tau\mid x)=1$, yields $ZP_\top(x)=R(x)$, recovering the desired reward-proportional terminal distribution.
The triplet $(P_F,P_B,Z)$ can be learned by minimizing the logarithmic residual of the TB condition. For a complete trajectory $\tau$ terminating at $x$, the trajectory balance loss is
\begin{equation}
    \mathcal L_{\mathrm{TB}}(\tau)
    =
    \left[
        \log
        \frac{
            ZP_F(\tau)
        }{
            R(x)P_B(\tau\mid x)
        }
    \right]^2.
    \label{eq:tb-objective}
\end{equation}
Thus, $\mathcal L_{\mathrm{TB}}(\tau)=0$ for every complete trajectory $\tau$ implies the TB condition in Eq.~\ref{eq:tb-condition}, and consequently the desired terminal distribution.

\section{Hierarchical Cooperative GFlowNets}
\label{sec:method}

In this section, we introduce HCG as a framework that coordinates multiple GFlowNets by dividing exploration hierarchically according to each model's terminal support. We assume for now that these supports are available through an oracle, deferring its practical instantiation as a classifier to Section~\ref{sec:amortized-ownership}, and show that the resulting specialized models can be combined to recover the global target.\looseness=-1

\subsection{Hierarchical division of exploration}

\begin{wrapfigure}{r}{0.55\textwidth}
    \centering
    \vspace{-1.5em}
    \includegraphics[width=\linewidth]{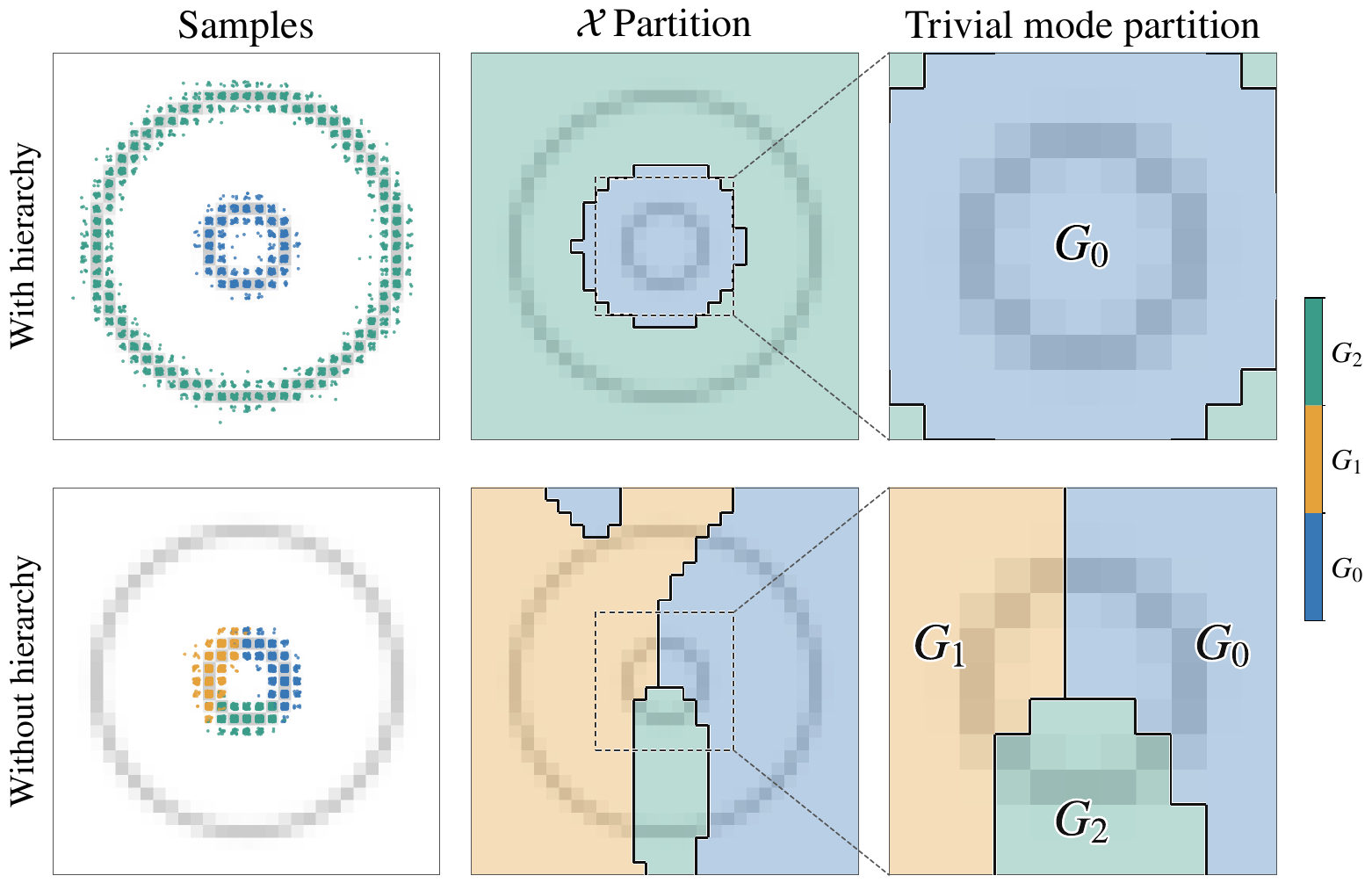}
    \caption{Three GFlowNets on a two-ring reward with (top) and without (bottom) an exploration hierarchy. \textbf{Left:} samples, colored by generating model. \textbf{Center:} regions of $\mathcal{X}$ owned by each model. \textbf{Right:} zoom on the inner mode. Without hierarchical coordination, all models split the inner mode and the outer ring is never sampled. With hierarchy, $G_0$ owns the inner mode and a lower-ranked model is redirected to the outer ring.}
    \label{fig:hierarchy-adv}
    \vspace{-2em}
\end{wrapfigure}

\label{sec:ownership}

An intuitive approach for orchestrating multiple GFlowNets is to partition the search space into different regions to be explored by different GFlowNets.
However, a naïve implementation of this approach may not guarantee broader exploration coverage, given that we could end up with different models \emph{competing} for the same region while leaving other regions unexplored. 

Figure~\ref{fig:hierarchy-adv} illustrates this failure mode. Without a hierarchy, the models may split the inner mode among themselves and leave the outer ring undiscovered. We therefore introduce a hierarchy in which regions covered by higher-ranked models are unavailable to lower-ranked ones. As the figure shows, the lower-ranked models are then redirected toward complementary regions, and the ensemble recovers both modes.\looseness=-1

Consider $K+1$ independently parameterized GFlowNets $\mathcal{G}=\{G_0,\ldots,G_K\}$, all sharing the same reward, and impose the hierarchy $G_0\succ G_1\succ\cdots\succ G_K$, where lower indices indicate higher priority.
The hierarchy resolves overlaps in coverage by assigning each terminal object to the highest-ranked model that covers it.
Letting $S_i$ denote the support of the terminal distribution of the forward policy of $G_i$, the ownership of each terminal is determined as follows: \looseness=-1
\begin{equation}
    o(x)=\min\{i:x\in S_i\},
    \qquad
    O_i
    = 
    S_i\setminus{\textstyle\bigcup_{j<i}} S_j.
    \label{eq:ownership}
\end{equation}
This definition guarantees that the ownership regions $\{O_i\}_{i=0}^K$
form
a partition of the overall support $\bigcup_i S_i$.
We also define the admissible region of $G_i$ as the complement of the union
of all higher-ranked ownership regions, \looseness=-1
\begin{equation}
    {\textstyle
    A_i
    =
    \mathcal{X}\setminus\bigcup_{j<i}O_j,
    \qquad
    D_i
    =
    \mathcal{X}\setminus A_i
    =
    \bigcup_{j<i}O_j.}
    \label{eq:admissible-disallowed}
\end{equation}
Thus, $A_i$ is the allowed region, whereas $D_i$ represents the disallowed region for a GFlowNet $G_i$. These regions are not fixed, and as each support $S_i$ evolves during training, so do $O_i$, $A_i$, and $D_i$, which we leave implicit in the notation for clarity.
As admissibility is distinct from the support, $A_i$ may contain terminals that $G_i$ has not yet discovered. 
In particular, for the highest-ranked model, $A_0=\mathcal{X}$ and $D_0=\emptyset$, so $G_0$ is allowed to explore the entire space.
Ideally, if the model discovers a new mode, that region immediately becomes part of $O_0$ and therefore is no longer admissible for any other GFlowNet.
This mechanism ensures models do not compete over the same modes but rather redirect their exploration toward undiscovered regions. \looseness=-1

Under the hierarchical allocation, each model with a nonempty ownership
region is assigned the reward-proportional distribution restricted to that region,\looseness=-1
\begin{equation}
    P^{i}_{\top}(x)
    =
    (R(x)/Z_i) \ind{x\in O_i},
    \quad
    Z_i
    =
    \textstyle\sum_{x\in O_i}R(x).
    \label{eq:component-target}
\end{equation}

\subsection{Combining specialized GFlowNets}
\label{sec:mixture}

Since each GFlowNet concentrates its probability mass in its ownership region, no single model samples from the global target. The global target must therefore be recovered by an ensemble of specialized models. We combine them by first selecting $G_i$ with probability $\omega_i = Z_i / \sum_{j=0}^{K}Z_j$ and then sampling from its forward policy. The resulting terminal distribution is \begin{equation}
P_{\mathrm{ens}}(x)=\textstyle\sum_{i=0}^{K}\omega_i P_F^{i,\top}(x).
\label{eq:ensemble}
\end{equation}
Models with $Z_i = 0$ receive zero weight and are excluded from the ensemble. The normalizers account for the different amounts of reward mass represented by the models. This is the same normalizer-weighted mixture used to sample from Boosted GFlowNets~\citep{dallantonia2026boostedgflownetsimprovingexploration}, where it combines sequentially trained members; here, we apply it to components trained in parallel under the hierarchy. 

\begin{proposition}[Target recovery]
    \label{prop:target-recovery}
    Suppose the ownership regions $O_0,\ldots,O_K$ are disjoint and cover $\mathcal{X}$ and each model with $Z_i > 0$ satisfies Eq.~\ref{eq:component-target}. Then the normalizer-weighted ensemble satisfies $P_{\mathrm{ens}}(x)=R(x)/Z$ for every $x\in\mathcal{X}$.
\end{proposition}

\emph{Proof.}
Under Eq.~\ref{eq:component-target}, disjointness and full coverage imply $\sum_i Z_i=Z$. Therefore,
\begin{equation}
    P_{\mathrm{ens}}(x)
    =
    \sum_{i:Z_i>0}
    \frac{Z_i}{\sum_j Z_j}
    \frac{R(x)\ind{x\in O_i}}{Z_i}
    =
    \frac{R(x)}{\sum_j Z_j}
    \sum_{i:Z_i>0}\ind{x\in O_i}
    =
    \frac{R(x)}{Z}.
    \label{eq:ensemble-target}
\end{equation}
The last identity follows because each terminal belongs to exactly one ownership region. 

The ensemble therefore preserves the target under local reward calibration and complete coverage. This result concerns the recombination of calibrated specialized models, rather than guaranteeing that the training dynamics achieve complete coverage or exact calibration.
Uniform model weights would not generally do so when the ownership regions contain different amounts of reward mass. 

If $G_0$ eventually covers all of $\mathcal{X}$, then $O_0=\mathcal{X}$ and the lower-ranked ownership regions are empty, resulting in asymptotic degeneracy. The normalizers and mixture weights therefore vanish, and the ensemble reduces to $G_0$ without changing the target.
Our objective is therefore complementary exploration during finite-time training, rather than persistent specialization or a nonzero mixture weight for every model.

\section{Training Hierarchical Cooperative GFlowNets}
\label{sec:coordination}
Realizing the hierarchical allocation of Section~\ref{sec:ownership} requires:
\begin{inparaenum}[i)]
    \item training each GFlowNet to move probability away from its disallowed region while preserving reward-proportional learning elsewhere, and
    \item inferring these regions dynamically during training, replacing the oracle assumed so far.
\end{inparaenum}
We address both in the following sections and combine them into a coordinated training procedure.

In practice, we interpret coverage through the empirical support induced by finite sampling. While Section 3 defines $S_i$ using the support of the terminal distribution, during training we only observe the subset of terminals visited by each GFlowNet within a finite sampling budget. This distinction is particularly relevant to exploration as even when a policy has full mathematical support, regions assigned negligible probability may remain effectively undiscovered throughout training.

\subsection{Learning Objectives under Disallowed Regions}
\label{sec:repelling}

Given that standard GFlowNet balance objectives require strictly positive rewards, a GFlowNet cannot move probability away from its disallowed region simply by assigning it zero target mass, as the objective would become ill-defined.
We therefore seek a loss function that preserves the original reward on admissible regions while driving probability away from regions that should no longer be explored by a given model.
Let $r(x)>0$ be an auxiliary reward function on an arbitrary GflowNet's disallowed region $D$. For a complete trajectory $\tau$ ending in $x$, we define the loss function
\begin{equation}
    \mathcal{L}(\tau; D, r)
    =
    \left[
        \log\left(
            \frac{
                ZP_F(\tau)
            }{
                \left[
                    R(x)\ind{x\notin D}
                    +
                    r(x)\ind{x\in D}
                \right]
                P_B(\tau\mid x)
            }
            +
            \ind{x\in D}
        \right)
    \right]^2.
    \label{eq:repelling-objective}
\end{equation}

For $x\notin D$, Eq.~\ref{eq:repelling-objective} reduces exactly to Trajectory Balance. For $x\in D$, the additional positive term inside the logarithm prevents balance from being achieved while the forward policy assigns nonzero probability to trajectories ending at $x$. This term enforces exclusion, while $r$ shapes the model's behavior on $D$ as long as mass remains there. Since disallowed regions are by construction covered by higher-ranked models, a lower-ranked model trained on $R$ would remain attracted to their already discovered modes. Replacing $R$ by $r$ on $D$ decouples it from the original reward landscape.

Proposition~\ref{prop:restricted-target} shows that, as long as part of the space remains allowed, at the optimum this construction removes $D$ from the terminal support and recovers the reward-proportional distribution on its complement, regardless of the choice of $r$.

\begin{proposition}[Restricted target]
    \label{prop:restricted-target}
    Suppose $D\subsetneq\mathcal{X}$ and the objective is enforced over all complete
    trajectories. Its global minimum is zero, with
    \begin{equation}
        P_\top^\star(x)
        \propto
        R(x)\ind{x\notin D},
        \qquad
        Z^\star
        =
        {\textstyle\sum_{x\notin D}R(x)}.
        \label{eq:restricted-target}
    \end{equation}
\end{proposition}

\begin{corollary} 
If $D=\mathcal{X}$ and $Z$ is learned jointly with the policies, no zero-loss solution exists for $Z>0$, and the objective's infimum is approached as $Z\rightarrow0$.
\end{corollary}

Note that when a GFlowNet places all its mass in $D$, $r$ shapes how the model behaves within this region. When the entire space is disallowed and $Z$ is fixed, Proposition~\ref{prop:auxiliary-reward-bias} shows that the terminal distribution becomes proportional to $r$, acting as the model's preference among disallowed terminals.\looseness-1

\begin{proposition}[Auxiliary reward bias]
    \label{prop:auxiliary-reward-bias}
    Suppose $D=\mathcal{X}$ and fix $Z>0$. Then the on-policy expected objective
    $\mathbb{E}_{\tau\sim P_F}\!\left[\mathcal{L}(\tau;\mathcal{X},r)\right]$
    is globally minimized when $P_\top^\star(x)
        \propto
        r(x).$
\end{proposition}

We therefore want to move the model away from disallowed regions, thus avoiding modes already found by higher-ranked models. With that in mind, we set $r(x)=1$, noting that any constant induces the same uniform distribution over $D$. 

When training multiple models, each GFlowNet $G_i$ given by the triplet $(P_F^i, P_B^i, Z_i)$ iteratively minimizes $\mathcal{L}(\tau; D_i, r)$ (Eq. \ref{eq:repelling-objective}). Since $D_0=\emptyset$, the highest-ranking model's loss reduces to standard Trajectory Balance over the entire terminal space, while for every $i>0$, the objective preserves the original reward-proportional target in its allowed region $A_i$, and the model is driven out of $D_i$.

\subsection{Mass-driven Hierarchical Ownership Inference}
\label{sec:amortized-ownership}

A GFlowNet $G_i$ owns a terminal $x$ if it is the highest-ranked component whose support contains $x$ (Section~\ref{sec:ownership}). 
Determining ownership thus requires checking whether any $G_j$, with $j<i$, assigns nonzero probability to that terminal. Although straightforward, evaluating these probabilities through each model's forward and backward policies becomes costly as the ensemble grows.
With $K+1$ GFlowNets generating $N$ terminals per training iteration, this approach requires $O(NK^2)$ cross-model queries, each estimated by Monte Carlo sampling from the backward policy, and places communication and cross-model evaluations directly on each model's training path.

Instead, we centralize coverage information in a shared coordinator. Each model communicates only with this module, which aggregates information from the ensemble and determines the owner of each sampled
terminal object. We implement the coordinator as a classifier $C_\phi(i\mid x)$, trained from generated samples to predict hierarchical ownership. \looseness=-1 

Unweighted, such a classifier would assign each terminal to whichever component samples it most often, penalizing models that spread their mass over broad regions. Scaling samples from $G_i$ by $Z_i$ compensates for differences in sampling concentration. To reflect the prescribed hierarchy when multiple models cover the same terminal, we also multiply by a priority $a_i > 0$, so each sample from $G_i$ contributes $Z_i a_i$. The ratio $a_i/a_{i+1}$ controls how strongly
the classifier favors one level of the hierarchy over the next when their coverage
overlaps. The loss of the classifier is therefore defined as \looseness=-1

\begin{equation}
    \mathcal{L}_C(\phi)
    =
    -
    {\textstyle \sum_{i=0}^{K}}
    Z_i a_i
    \,
    \mathbb{E}_{X\sim P_\top^i}
    \left[
        \log C_\phi(i\mid X)
    \right].
    \label{eq:classifier-loss}
\end{equation}

Since $Z_i$ is also a learnable parameter, the correction weight is approximate early in training.
When the classifier has not yet resolved an overlap in favor of the higher-ranked model, both GFlowNets continue to be trained with standard Trajectory Balance in that region. As their TB losses decrease, $Z_iP_\top^i(x)$ becomes increasingly aligned with $R(x)$, allowing the priority weights to resolve ownership and activate
repulsion in the lower-ranked model. \looseness=-1

\begin{proposition}[Ownership recovery]
    \label{prop:ownership-recovery}
    Let $x$ be a terminal state covered by at least one GFlowNet. Suppose that every
    GFlowNet whose support contains $x$ has zero Trajectory Balance loss on all
    trajectories terminating at $x$. If $a_0>a_1>\cdots>a_K>0$, then the
    optimum $C^\star$
    of Eq.~\ref{eq:classifier-loss} satisfies \looseness=-1
    \begin{equation}
        \textstyle\arg\max_i C^\star(i\mid x)
        =
        o(x).
        \label{eq:classifier-recovers-ownership}
    \end{equation}
\end{proposition}

Indeed, local Trajectory Balance implies
$Z_iP_\top^i(x)=R(x)$ for every GFlowNet whose support contains $x$, so
the relative classifier weights on $x$ are determined only by the priority weights
$a_i$. The proof of Proposition~\ref{prop:ownership-recovery} is deferred
to Appendix \ref{app:proofs}. \looseness=-1

\subsection{Coordinated Training}
\label{sec:coordinated-training}
The ownership classifier turns the collected global coverage information into a local training signal for each GFlowNet. For a terminal $x$, we
define the predicted owner as
\begin{equation}
    o_\phi(x)
    =
    \textstyle\arg\max_j C_\phi(j\mid x).
    \label{eq:predicted-owner}
\end{equation}

\begin{wrapfigure}[16]{r}{0.45\textwidth}
    \vspace{-1em}
    \begin{minipage}{\linewidth}
        \small
        \hrule
        \vspace{0.5em}
        \captionsetup{type=algorithm}
        \captionof{algorithm}{HCG training}
        \label{alg:hcgfn-training}
        \vspace{-0.75em}
        \hrule
        \vspace{0.4em}
        \begin{algorithmic}[1]
            \Require $\{G_i\}_{i=0}^K$, $C_\phi$,
            $\{a_i\}_{i=0}^K$, batch size $N$,
            coordination period $T_C$
            \State $iter \gets 0$
            \While{not converged}
            \ForAll{$i=0,\ldots,K$ \textbf{ in parallel}}
            \State Sample batch $\mathcal{B}_i$ from $G_i$
            \State Update $G_i$ on $\mathcal{B}_i$
            using $\ell_i(\,\cdot\,;C_\phi)$
            \State Send $(\mathcal{B}_i,i,Z_i)$
            to coordinator
            \EndFor
            \If{$iter \bmod T_C = 0$}
            \State Update $C_\phi$ on $\{\mathcal{B}_i\}$ weighted by $Z_i a_i$
            \EndIf
            \State $iter \gets iter + 1$
            \EndWhile
        \end{algorithmic}
        \vspace{0.3em}
        \hrule
    \end{minipage}
    \vspace{-0.6em}
\end{wrapfigure}

For $G_i$, a terminal is treated as disallowed whenever its predicted owner is ranked above $i$, which defines the disallowed region $D_i=\{x : o_\phi(x)<i\}$, and we write $\ell_i(\tau;C_\phi)=\mathcal{L}(\tau;D_i,r)$.
Under the objective of Section~\ref{sec:repelling}, $G_i$ follows standard TB on terminal objects not assigned above it in the hierarchy and is driven out of those that are.

Conditional on the current classifier, each GFlowNet can be trained independently.
Each model samples trajectories from its own policy, queries the classifier for ownership, and updates locally without accessing any other GFlowNet.
Components also provide their generated terminals, model indices, and normalizers to the coordinator, which updates the classifier using the $Z_i a_i$ weighting (Section~\ref{sec:amortized-ownership}). The complete training procedure is summarized in Algorithm~\ref{alg:hcgfn-training}.

\section{Experiments}
\label{sec:results}

We evaluate HCG across several GFlowNet benchmarks to assess whether hierarchical coordination improves exploration efficiency and the discovery of diverse high-reward states compared to baseline methods designed to encourage exploration. We further evaluate whether increasing the number of components leads to broader mode coverage.

As baselines, we consider a single GFlowNet trained with trajectory balance (TB) as a reference; Sibling-augmented GFlowNets (SA)~\citep{madan2025sibling}, which pair the main GFlowNet with a sibling explorer trained with a random network distillation bonus; and Adaptive Complementary Exploration (ACE)~\citep{dallantonia2026avoid}, whose explorer is steered toward regions the main GFlowNet undersamples through a divergent TB objective.
We evaluate each benchmark using its task-specific metrics (details in Appendix~\ref{app:diffgrid}, ~\ref{appendix:qknapsack}, ~\ref{app:sequences} and ~\ref{appendix:amp}).

For DiffGrid (Sec.~\ref{ssec:diff_grid}), we measure the total variation (TV) distance between the learned terminal distribution $p_T$, estimated by backward importance sampling and normalized over the finite terminal state space, and the normalized target distribution $\pi(x) \propto R(x)$: $\mathrm{TV}(p_T,\pi) = \nicefrac{1}{2}\sum_{x\in\mathcal{X}} \left|p_T(x)-\pi(x)\right|$.
For Quadratic Knapsack (Sec.~\ref{sec:q_knapsack}) and Bit Sequences (Sec.~\ref{sec:bit_sequences}), we report the top-$K$ average reward among the unique states discovered throughout training. For AMP discovery (Sec.~\ref{ssec:peptides}), we report the cumulative count of effective peptides found over the course of training. All metrics are computed on samples from the primary GFlowNet for baselines and from the normalizer-weighted ensemble described in Section~\ref{sec:mixture} for HCG.
An ablation on the classifier update frequency and a comparison under a matched wall-clock budget are provided in Appendices~\ref{appendix:ablation_classifier} and~\ref{appendix:wall-clock}.

\subsection{DiffGrid}
\label{ssec:diff_grid}

\paragraph{Setup.}
DiffGrid tests whether hierarchical coordination improves exploration while preserving the reward-proportional target. An agent moves along either coordinate of a two-dimensional lattice or stays in place, terminating after a fixed number of steps $T=80$. Including time in the state makes the graph acyclic despite spatial revisits. We report four reward landscapes: \emph{8 Gaussians}, \emph{Rings}, \emph{Moons}, and \emph{Spiral} on an $81 \times 81$ grid across five seeds, with a budget of 10{,}000 training iterations per run. Additional reward definitions and implementation details appear in Appendix~\ref{app:diffgrid}.

\paragraph{Results.}
Table~\ref{tab:diffgrid:size40} and Figure~\ref{fig:diffgrid:epochs} show the final TV distance and its evolution throughout training, respectively, with our method using 8 components.
The gains from hierarchical coordination are largest on landscapes with well-separated modes. On 8 Gaussians and Rings, HCG substantially reduces the TV distance relative to all baselines, and continues to improve through the end of training. On Moons and Spiral, the reward mass forms connected structures rather than isolated modes, so a single sampler can traverse much of the support and coordination has less to add. Here, although ACE reduces the distance faster early in training, HCG closes the gap and performs on par with it by the end, while remaining well below TB and SA.

We further show, in Figure~\ref{fig:diffgrid:tvandsamples}, that scaling the number of components translates directly into broader coverage on 8 Gaussians. While a single GFlowNet collapses onto one mode, each additional component is steered toward modes not yet covered, so the number of discovered modes grows with the size of the ensemble. The classifier partitions the space accordingly, allowing each GFlowNet to own different modes rather than competing for the same ones. The TV distance decreases consistently as components are added, supporting our central claim that, under hierarchical coordination, additional compute becomes broader exploration.

\begin{table}[h]
    \centering
    \caption{Final total variation distance ($\downarrow$) on DiffGrid ($81\times81$ grid), reported as mean $\pm$ sample standard deviation over five seeds. Bold indicates the lowest mean in each column. HCG achieves the lowest distance on three of the four landscapes, and remains competitive on Spiral.} %
    \label{tab:diffgrid:size40}
    \small
    \setlength{\tabcolsep}{4pt}
    \begin{tabular}{lcccc}
        \toprule
        Method & 8 Gaussians & Rings & Moons & Spiral \\
        \midrule
        TB 
        & $0.8750 \pm 0.0000$
        & $0.6143 \pm 0.2650$
        & $0.5590 \pm 0.1027$
        & $0.3446 \pm 0.1217$ \\
        SA
        & $0.8241 \pm 0.0939$
        & $0.6894 \pm 0.1492$
        & $0.5071 \pm 0.2526$
        & $0.3345 \pm 0.1218$ \\
        ACE
        & $0.4946 \pm 0.3282$
        & $0.5994 \pm 0.3190$
        & $0.1795 \pm 0.1855$
        & $\mathbf{0.1057 \pm 0.0635}$ \\
        \textbf{HCG (ours)}
        & $\mathbf{0.3054 \pm 0.1249}$
        & $\mathbf{0.3324 \pm 0.2637}$
        & $\mathbf{0.1703 \pm 0.1534}$
        & $0.1340 \pm 0.0873$ \\
        \bottomrule
    \end{tabular}
\end{table}

\begin{figure}[h]
    \centering
    \includegraphics[width=\linewidth]{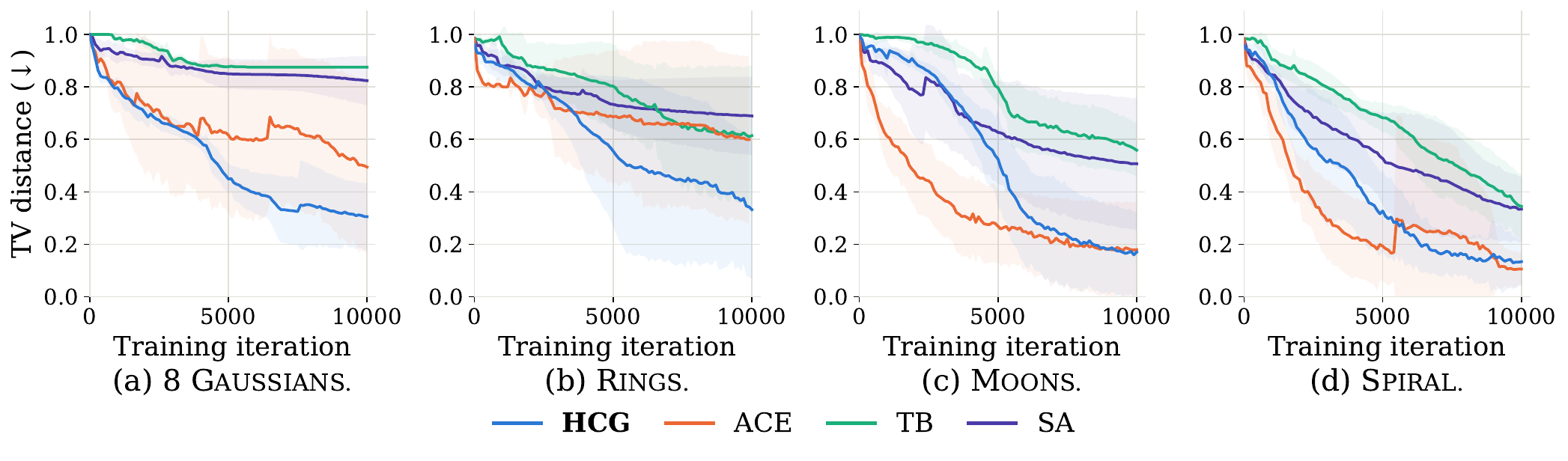}
    \vspace{-2em}
    \caption{Total variation distance ($\downarrow$) to the target distribution over training on DiffGrid rewards ($81\times 81$ grid). Lines show the mean over five seeds and shaded bands $\pm 1$ std. On 8 Gaussians and Rings, HCG continues to improve throughout training and ends with the lowest TV distance. On Moons and Spiral, whose reward mass forms connected structures rather than isolated modes, HCG remains competitive with ACE, the strongest baseline.}
    \vspace{-1.5em}
    \label{fig:diffgrid:epochs}
\end{figure}

\begin{figure}[t]
    \centering
    \includegraphics[width=\linewidth]{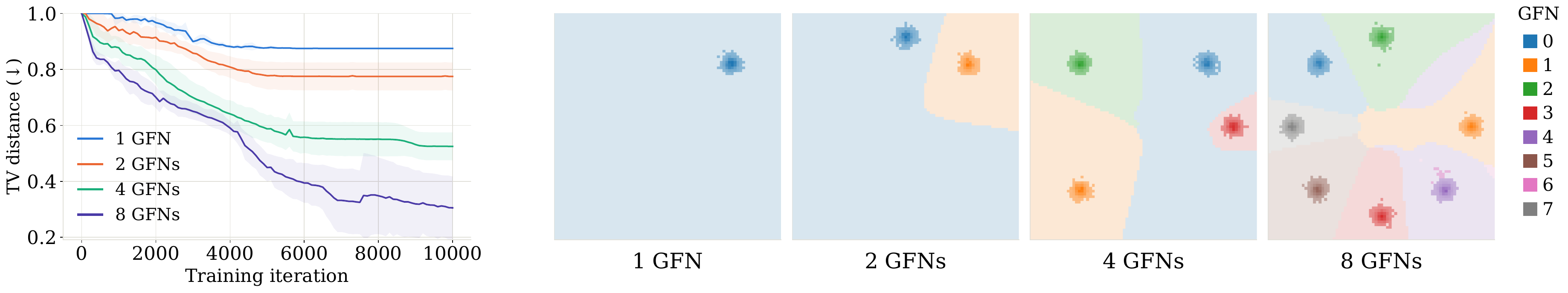}
    \caption{Effect of the number of GFlowNets on the DiffGrid 8~Gaussians reward on an $81\times81$ grid. \textbf{Left:} TV distance ($\downarrow$) to the target distribution over training. Lines show the mean over five seeds and bands $\pm 1$ std. \textbf{Right:} final samples for ensembles of 1, 2, 4, and 8 GFlowNets. Each sample is colored by the GFlowNet that generated it, and the background hue shows the region the classifier assigns to each GFlowNet. While a single GFlowNet collapses onto one mode, each additional component is steered toward a mode not yet covered, so that the number of discovered modes grows with the ensemble and the TV distance decreases accordingly.}
    \label{fig:diffgrid:tvandsamples}
    \vspace{-1em}
\end{figure}

\subsection{Quadratic Knapsack}
\label{sec:q_knapsack}

\paragraph{Setup.}
The Quadratic Knapsack Problem (QKP) is an NP-hard combinatorial optimization task that extends the classic 0-1 knapsack problem with pairwise dependencies among selected elements \citep{Caprara1999}. 
Given a set of $n$ items, let $x \in \{0, 1\}^n$ denote the binary decision vector, where $x_i = 1$ if item $i$ is included in the knapsack\footnote{$1 \leq x_i \leq L$ when multiple copies of items are allowed.}.
Each item $i$ has a weight $w_i > 0$, subject to a knapsack capacity constraint $C > 0$.
The objective is parameterized by a linear profit vector $p \in \mathbb{R}^n$ and a symmetric interaction matrix $P \in \mathbb{R}^{n \times n}$, where $P_{ij}$ represents the complementary or substitutive utility of selecting items $i$ and $j$ simultaneously ($i \neq j$).
The QKP task provides a challenging benchmark where trajectories correspond to sequential item insertions until the weight capacity is reached.
Reward definitions and implementation details appear in Appendix \ref{appendix:qknapsack}.

\begin{figure}[t]
    \centering
    \includegraphics[width=\linewidth]{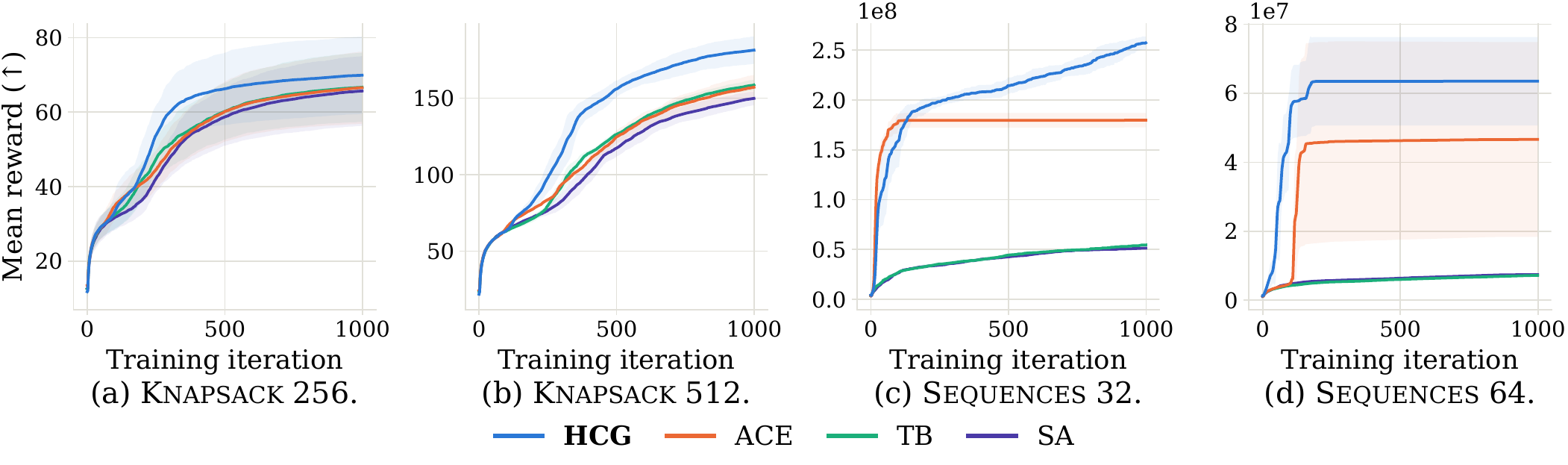}
    \caption{Top-200 mean reward ($\uparrow$) over two environments using two configurations:
  the knapsack task at $n=256$ and $n=512$, and the multi-modal bit-sequence task at $M=32$ and $M=64$.
  Results are the mean of 3 seeds and bands $\pm 1$ std.
  HCG reaches the highest top-200 reward in all four conditions, and its margin over the strongest baseline widens as the problem grows. This is more pronounced in the bit-sequences task ($M = 64$), where most baselines stall below $10^7$, ACE plateaus and HCG reaches more than $6\times10^7$ mean reward.}
    \label{fig:top_200}
    \vspace{-1em}
\end{figure}

\paragraph{Results.}
Fig. \ref{fig:top_200} (a) and (b) show the mean top-200 rewards maintained over all trajectories a method generates along 1000 training iterations, with $n$ in $\{256, 512\}$, respectively.
HCG exhibits the best iteration efficiency in the search for high-valued feasible solutions for the problem, when compared to the baselines.
In the more complex instance (512 items), the gap between HCG and the contenders is more pronounced, highlighting the exploration advantages of our proposal in this scenario.\looseness-1

\subsection{Bit Sequences}
\label{sec:bit_sequences}

\paragraph{Setup.}
We consider the bit sequence environment introduced by \citet{malkin2022trajectory}.
Fix a target length $M$ and let $\mathcal{X}=\{0,1\}^{M}$ be the set of terminal states and $\mathcal{S}=\bigcup_{m=0}^{M-1}\{0,1\}^{m}$ the set of intermediate ones.
Generation starts from the empty string $s_o=[]$ and appends a single bit to the right of the current prefix at each step.
The action set is $\{0,1\}$ at every state, no stop action is available, and generation halts as soon as the prefix reaches length $M$.
Every trajectory consequently has length exactly $M$, and the transition graph is a complete binary tree: each terminal state is reached by a unique path, so the backward policy is deterministic and carries no free parameters.
Mode generation and reward definitions appear in Appendix \ref{app:sequences}
We report results for $M \in \{32,64\}$, both multiples of $8$, so the block construction applies without truncation or padding.

\paragraph{Results.}
For bit sequences of length $32$ and $64$, the mean top-200 reward achieved by each method over 1000 iterations is shown in Fig.~\ref{fig:top_200}~(c) and (d), respectively.
In the more complex setting (sequences of 64 bits), HCG rapidly converges to a substantially higher mean reward, followed by ACE, while the baseline methods exhibit more limited progress.
Additionally, in the sequences of 32 bits task, HCG sustains steady improvement throughout the training process, successfully avoiding the premature plateau observed in ACE baseline.\looseness=-1

\subsection{Antimicrobial Peptide (AMP) Discovery}
\label{ssec:peptides}

\begin{wrapfigure}[17]{r}{0.53\textwidth}
    \vspace{-1.5em}
    \centering
    \includegraphics[width=\linewidth]{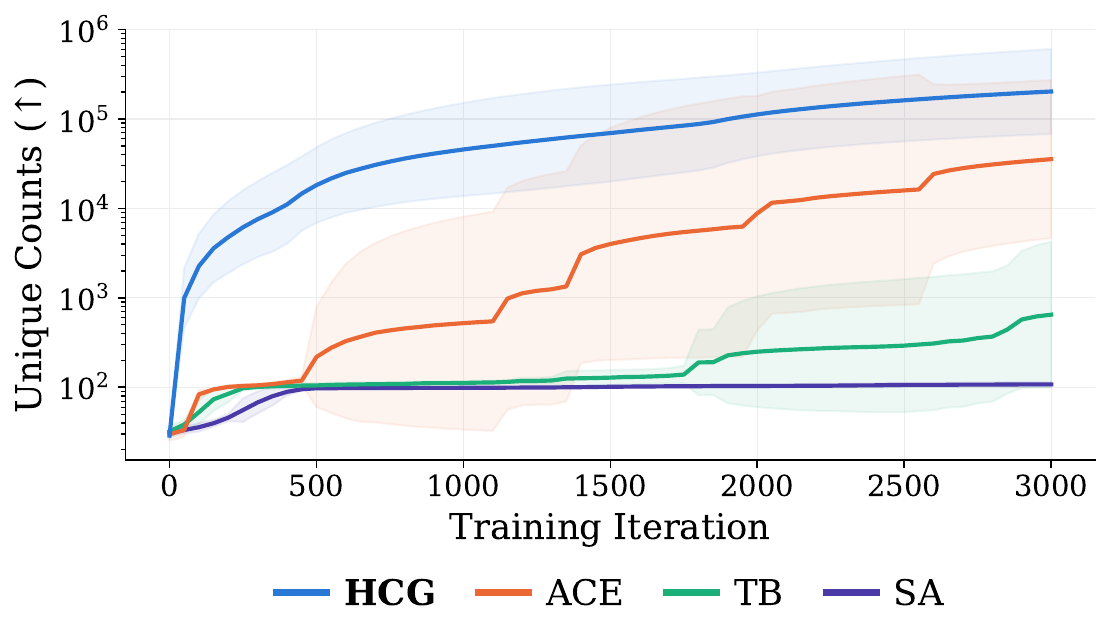}
    \caption{Cumulative count of unique high-scoring peptides discovered per method throughout training (log scale). Lines show the mean over 5 seeds, and bands span $\pm 1$ std. HCG shows faster discovery and sustained exploration over training.}
    \label{fig:peptides}
    \vspace{-0.5em}
\end{wrapfigure}

\paragraph{Setup.}
We also evaluate HCG in the task of AMP design \citep{trabucco2022designbench, jain2022biological}, where we intend to generate peptides with antimicrobial properties. Candidate peptides have variable length from 1 to 10, drawn from a vocabulary of size 20, resulting in a search space of approximately $10^{13}$ candidates. The reward is the maximum probability across an ensemble of five Random Forest classifiers, each estimating activity against a distinct pathogen. To avoid overweighting top-scoring sequences, we clip the log-reward at zero for peptides with predicted probability $\geq 0.95$. Further details are provided in Appendix \ref{appendix:amp}.

\paragraph{Results.}
Figure~\ref{fig:peptides} shows that HCG with four GFlowNets discovers substantially more unique effective peptides than all other methods. While TB and SA quickly saturate at a small number of peptides, consistent with samplers that collapse onto a few high-reward regions, HCG surpasses their final counts within the first hundred iterations. ACE avoids this early saturation, but progresses in abrupt jumps, suggesting that its explorer occasionally uncovers new regions without sustaining their exploration. HCG, in contrast, follows its fast initial phase with steady discovery throughout training and concludes well ahead of ACE, as each lower-ranked component is continually pushed away from regions already covered by those above it. Finally, HCG varies less across seeds than ACE, whose band spans several orders of magnitude, indicating more consistent discovery.

\section{Conclusion}

We have introduced Hierarchical Cooperative GFlowNets, a method that trains several GFlowNets in parallel under a priority hierarchy, coordinated through a shared ownership classifier.
This design turns additional compute into complementary rather than redundant exploration without requiring a decomposition of the reward or the state graph, and its normalizer-weighted ensemble provably recovers the reward-proportional target under mild assumptions.
Our experiments across synthetic, combinatorial, and biological sequence design tasks show that HCG outperforms existing exploration methods in most settings, with the clearest gains on sparse rewards whose modes lie far from one another.\looseness=-1

Although a constant auxiliary reward within disallowed regions already yields strong results, the design of a reward that actively accelerates a component's escape toward admissible regions is a promising direction for future work. Likewise, incorporating inductive biases suited to the structure of the environment into the ownership classifier, and analyzing their empirical and theoretical implications, is a natural next step. Overall, we believe HCG's readiness for distributed computing architectures may be especially useful for large-scale applications such as drug discovery and biological sequence design, where finding many diverse high-reward candidates is the goal.\looseness=-1

\looseness-1

\subsection*{Acknowledgments}

We acknowledge the support of the Funda\c{c}\~{a}o Carlos Chagas Filho de Amparo \`{a} Pesquisa do Estado do Rio de Janeiro (FAPERJ) (\mbox{SEI-260003/020348/2025}, \mbox{SEI-260003/020694/2025}) and the Conselho Nacional de Desenvolvimento Cient\'{i}fico e Tecnol\'{o}gico (CNPq) (\mbox{404336/2023-0}, \mbox{305692/2025-9}, \mbox{445170/2024-7}). E. da Silva's work is funded by Portuguese national funds through FCT~-- Foundation for Science and Technology, I.P., within the scope of the research unit \mbox{UID/00326}~- Centre for Informatics and Systems of the University of Coimbra.

\bibliographystyle{iclr2027_conference}
\bibliography{iclr2027_conference}

\appendix

\section{Proofs}
\label{app:proofs}

\subsection{Proof of Proposition~\ref{prop:restricted-target}}
\label{app:proof-restricted-target}

Since $\mathcal{L}(\tau;D,r)\geq 0$, it is enough to characterize when the
loss can vanish. For a trajectory terminating at $x\in D$,
Equation~\ref{eq:repelling-objective} gives
\[
    \mathcal{L}(\tau;D,r)
    =
    \left[
        \log\left(
            1+
            \frac{ZP_F(\tau)}
                 {r(x)P_B(\tau\mid x)}
        \right)
    \right]^2.
\]
Hence, zero loss requires $P_F(\tau)=0$ for every trajectory terminating
in $D$. For $x\notin D$, the objective reduces to Trajectory Balance, and
zero loss requires
\[
    ZP_F(\tau)
    =
    R(x)P_B(\tau\mid x).
\]
Summing over all trajectories terminating at $x$ and using
$\sum_{\tau:\tau\to x}P_B(\tau\mid x)=1$ yields
\begin{equation}
    P_\top(x)=0
    \quad \text{for } x\in D,
    \qquad
    ZP_\top(x)=R(x)
    \quad \text{for } x\notin D.
    \label{eq:restricted-target-proof}
\end{equation}
Since $P_\top$ is normalized,
\[
    1
    =
    \sum_{x\notin D}P_\top(x)
    =
    \frac{1}{Z}\sum_{x\notin D}R(x),
\]
which gives $Z=Z^\star=\sum_{x\notin D}R(x)$ and therefore
$P_\top^\star(x)\propto R(x)\ind{x\notin D}$. Conversely, the standard
Trajectory Balance construction on this restricted target satisfies the
conditions above and attains zero loss. This proves the result.

\subsection{Proof of the corollary to Proposition~\ref{prop:restricted-target}}
\label{app:proof-vanishing-normalizer}

When $D=\mathcal{X}$, every complete trajectory terminates in the
disallowed region, so
\[
    \mathcal{L}(\tau;\mathcal{X},r)
    =
    \left[
        \log\left(
            1+
            \frac{ZP_F(\tau)}
                 {r(x)P_B(\tau\mid x)}
        \right)
    \right]^2.
\]
For any $Z>0$, a normalized forward policy assigns positive probability
to at least one complete trajectory, for which the quantity inside the
logarithm is strictly greater than one. Therefore, no zero-loss solution
exists for $Z>0$. On the other hand, for fixed policies,
$\mathcal{L}(\tau;\mathcal{X},r)\rightarrow 0$ as $Z\rightarrow0$ for
every trajectory carrying forward probability. Hence the infimum is
approached as $Z\rightarrow0$.

\subsection{Proof of Proposition~\ref{prop:auxiliary-reward-bias}}
\label{app:proof-auxiliary-reward-bias}

For $D=\mathcal{X}$, the on-policy expected objective is
\[
    \mathbb{E}_{\tau\sim P_F}
    \left[
        \log^2\left(
            1+
            \frac{ZP_F(\tau)}
                 {r(x)P_B(\tau\mid x)}
        \right)
    \right].
\]
Consider first the trajectories terminating at a fixed $x$ with
$P_\top(x)>0$. Writing
$P_F(\tau)=P_\top(x)P_F(\tau\mid x)$, their contribution to the
expectation is
\begin{align}
    P_\top(x)
    \sum_{\tau:\tau\to x}
    P_B(\tau\mid x)
    \frac{P_F(\tau\mid x)}{P_B(\tau\mid x)}
    \log^2\left(
        1+
        \frac{ZP_\top(x)}{r(x)}
        \frac{P_F(\tau\mid x)}{P_B(\tau\mid x)}
    \right).
    \label{eq:auxiliary-proof-conditional}
\end{align}
For every $c>0$, the function
$u\mapsto u\log^2(1+cu)$ is strictly convex on $\mathbb{R}_{+}$, since
\[
    \frac{\mathrm{d}^2}{\mathrm{d}u^2}
    \left[
        u\log^2(1+cu)
    \right]
    =
    \frac{
        2c\left[
            cu\left(1+\log(1+cu)\right)
            +2\log(1+cu)
        \right]
    }{
        (1+cu)^2
    }
    >0
\]
for $u>0$. Moreover,
\[
    \sum_{\tau:\tau\to x}
    P_B(\tau\mid x)
    \frac{P_F(\tau\mid x)}{P_B(\tau\mid x)}
    =
    1.
\]
Jensen's inequality applied to
Equation~\ref{eq:auxiliary-proof-conditional} therefore shows that its
minimum, for fixed $P_\top(x)$, is attained when
$P_F(\tau\mid x)=P_B(\tau\mid x)$, and is equal to
\[
    P_\top(x)
    \log^2\left(
        1+\frac{ZP_\top(x)}{r(x)}
    \right).
\]

It remains to optimize the terminal probabilities. The resulting
objective is
\[
    \sum_{x\in\mathcal{X}}
    r(x)
    \left[
        \frac{P_\top(x)}{r(x)}
        \log^2\left(
            1+
            Z\frac{P_\top(x)}{r(x)}
        \right)
    \right].
\]
Applying Jensen's inequality again, now with weights proportional to
$r(x)$, equality holds if and only if $P_\top(x)/r(x)$ is constant over
$\mathcal{X}$. Since $P_\top$ is normalized,
\[
    P_\top^\star(x)
    =
    \frac{r(x)}
         {\sum_{x'\in\mathcal{X}}r(x')},
\]
which proves $P_\top^\star(x)\propto r(x)$.

\subsection{Proof of Proposition~\ref{prop:ownership-recovery}}
\label{app:proof-ownership-recovery}

For a fixed terminal $x$, Equation~\ref{eq:classifier-loss} is a weighted
cross-entropy whose class weights are
$Z_i a_i$. Its population optimum therefore satisfies
\begin{equation}
    C^\star(i\mid x)
    =
    \frac{
        Z_i a_i P_\top^i(x)
    }{
        \sum_{j=0}^{K}
        Z_j a_j P_\top^j(x)
    }.
    \label{eq:classifier-optimum-proof}
\end{equation}
For every GFlowNet whose support contains $x$, zero Trajectory Balance
loss on all trajectories terminating at $x$ implies
\[
    Z_iP_F^i(\tau)
    =
    R(x)P_B^i(\tau\mid x).
\]
Summing over those trajectories gives
$Z_iP_\top^i(x)=R(x)$. Models whose support does not contain $x$ instead
have $P_\top^i(x)=0$. Equation~\ref{eq:classifier-optimum-proof} thus
reduces, among the models covering $x$, to
\[
    C^\star(i\mid x)
    =
    \frac{a_i}
         {\sum_{j:x\in S_j}a_j}.
\]
Since $a_0>a_1>\cdots>a_K>0$, the largest probability is assigned to the
smallest index whose support contains $x$. By
Equation~\ref{eq:ownership}, this index is exactly $o(x)$, and therefore
\[
    \arg\max_i C^\star(i\mid x)=o(x).
\]

\section{Experiment Details}

\subsection{DiffGrid}
\label{app:diffgrid}

\paragraph{Environment and rewards.} We use a two-dimensional lattice with half-width $m$ and horizon $T=2m$, with actions that increment or decrement one coordinate, or leave the position unchanged. We write states as $(u,t)$ with $u\in\{-m,\ldots,m\}^2$ and initialize the process at $(u,t)=((0,0),0)$. Each action advances the clock, boundary-crossing actions are masked, and termination occurs at $t=T$, so every trajectory contains exactly $2m$ transitions. Time augmentation makes the construction graph acyclic despite allowing spatial revisits. The terminal object $x$ is the final lattice position.

The 8 Gaussians and Rings landscapes are based on \citet{dallantonia2026avoid}, with the following implementation-specific scales. We define the unnormalized landscape densities
\begin{align}
    D_{\mathrm{rings}}(x)
    &= 10\sum_{r\in\{0.2m,\,0.8m\}}
       \exp\!\left[-\frac{(\lVert x\rVert_2-r)^2}{2\sigma^2}\right],
       \label{eq:diffgrid-rings}\\
    D_{\mathrm{8G}}(x)
    &= \sum_{j=0}^{7}
       \exp\!\left[-\frac{\lVert x-\mu_j\rVert_2^2}{2\sigma^2}\right],
       \label{eq:diffgrid-gaussians}
\end{align}
where $\sigma=m/30$ and $\mu_j=0.8m\bigl(\cos(2\pi j/8),\sin(2\pi j/8)\bigr)$. Thus, at $m=18$, the Gaussian width and ring thickness are both $0.6$.

We additionally evaluate Moons and Spiral rewards, each constructed as a sum of Gaussian kernels with the same width $\sigma=m/30$. For Moons, let $r=0.6m$, $\delta=0.5r$, $g=0.3r$, and $\theta_k=\pi k/255$ for $k=0,\ldots,255$. The anchors are
\begin{align}
    \mu_k^{(1)} &= (-\delta+r\cos\theta_k,\ r\sin\theta_k),\\
    \mu_k^{(2)} &= (\delta+r\cos(\pi+\theta_k),\ -g+r\sin(\pi+\theta_k)),
\end{align}
and the density is
\begin{equation}
    D_{\mathrm{moons}}(x)
    =\sum_{q=1}^{2}\sum_{k=0}^{255}
      \exp\!\left[-\frac{\lVert x-\mu_k^{(q)}\rVert_2^2}{2\sigma^2}\right].
\end{equation}
This gives 256 anchors on an upper semicircle centered at $(-\delta,0)$ and 256 on a lower semicircle centered at $(\delta,-g)$, for 512 equally weighted kernels in total.

For Spiral, let $b=0.85m/(4\pi)$ and $\mu(\theta)=b\theta(\cos\theta,\sin\theta)$ for $\theta\in[0,4\pi]$. We place 512 anchors uniformly in arc length along this curve,
\begin{equation}
    \begin{gathered}
    S(\theta)=\frac{b}{2}\left[\theta\sqrt{1+\theta^2}
                   +\operatorname{asinh}\theta\right],\\
    S(\theta_k)=\frac{k}{511}S(4\pi),
    \qquad k=0,\ldots,511,
    \end{gathered}
\end{equation}
where the implementation numerically inverts $S$. The density is
\begin{equation}
    D_{\mathrm{spiral}}(x)
    =\sum_{k=0}^{511}
      \exp\!\left[-\frac{\lVert x-\mu(\theta_k)\rVert_2^2}{2\sigma^2}\right].
\end{equation}
The spiral has one arm making two counterclockwise turns, starting at the origin in the positive horizontal direction and ending at radius $0.85m$. Neither anchor-based density is divided by its number of kernels.

For each landscape, the reward used in training is
\begin{equation}
    R(x)=\exp(1)\bigl[(1-\lambda)D(x)+\lambda\bigr],
    \qquad \lambda=10^{-6}.
\end{equation}
Here we omit the numerical $10^{-30}$ stabilizer added to the kernel sums. The positive floor preserves full target support. The factor $\exp(1)$ corresponds to an additive offset of one in log reward and cancels from the normalized target; the Rings factor of ten is applied before adding the floor. The four landscapes test isolated peaks, separated radial modes, curved modes, and extended winding structure, respectively.

\begin{figure}[t]
    \centering
    \includegraphics[width=\linewidth]{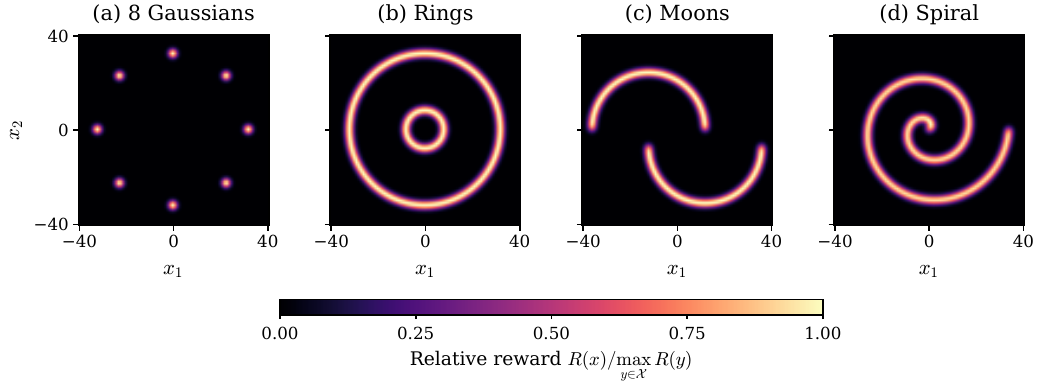}
    \caption{DiffGrid reward landscapes on an $81\times81$ grid ($m=40$). Each cell shows the implemented reward divided by its maximum within that landscape, using a shared linear color scale.
    (a) Eight Gaussian peaks lie on a circle of radius $0.8m$.
    (b) Two rings have radii $0.2m$ and $0.8m$.
    (c) Two offset semicircular arcs are represented by 256 Gaussian anchors each.
    (d) A single spiral arm makes two turns, with 512 anchors spaced uniformly in arc length and outer radius $0.85m$.
    All landscapes use width $\sigma=m/30=4/3$ and floor parameter $\lambda=10^{-6}$.}
    \label{fig:diffgrid-rewards}
\end{figure}

\paragraph{Comparisons and training protocol.}
We compare against TB \citep{malkin2022trajectory}, ACE \citep{dallantonia2026avoid}, and SA \citep{madan2025sibling}. HCG uses eight components by default, each with a forward policy of two hidden layers of 128 ReLU units and a linear backward policy. Both policies take the position, normalized time, and sine and cosine features at eight time frequencies as inputs, with invalid actions masked. The ownership classifier takes the terminal position as input and has two hidden layers of 512 ReLU units. We use AdamW with learning rates of $5\times10^{-3}$ for the policies, $5\times10^{-2}$ for the log-normalizers, and $10^{-3}$ for the classifier, and train for 10{,}000 epochs with 1{,}024 trajectories and one update per component per epoch. Training actions are sampled from the forward policy mixed with a uniform distribution over valid actions, with $\epsilon=0.1$.

Ownership is assigned by classifier argmax with priority decay $a=0.5$. After each iteration, the classifier receives ten updates on balanced minibatches of 2{,}048 examples drawn from a replay buffer of 65{,}536 examples, divided equally across components, with each example storing the log-normalizer of the model that generated it. Results use four or eight components and five seeds (40--44). ACE and SA use one hidden layer of 64 ReLU units in each policy.

\paragraph{Evaluation.}
We evaluate every 100 epochs. Although the finite time-augmented graph permits exact forward dynamic programming, the current evaluator estimates terminal marginals by backward importance sampling. For each component and each lattice terminal $x$, it samples $B=15$ backward trajectories and computes
\begin{equation}
    \widehat p_i(x)=\frac{1}{B}\sum_{b=1}^{B}
       \frac{P_F^i(\tau_b)}{P_B^i(\tau_b\mid x)},
    \qquad \tau_b\sim P_B^i(\cdot\mid x).
\end{equation}
These estimates are normalized over the lattice before forming the ensemble. For our method, the distribution being evaluated is the normalizer-weighted ensemble $P_{\mathrm{ens}}$ in Section~\ref{sec:mixture}, rather than an individual component or a uniform mixture. Its total variation distance to the target is
\begin{equation}
    \operatorname{TV}(P_{\mathrm{ens}},\pi)
    =\frac{1}{2}\sum_{x\in\mathcal{X}}
      \left|P_{\mathrm{ens}}(x)-\frac{R(x)}{\sum_{y\in\mathcal{X}}R(y)}\right|.
    \label{eq:diffgrid-tv}
\end{equation}
The logged TV substitutes the normalized Monte Carlo estimates for the component terminal marginals, so it is an estimated rather than exact TV.

\subsection{Quadratic Knapsack Experiment Details}
\label{appendix:qknapsack}

\paragraph{Environment Generation.} Using the variables defined in Section \ref{sec:q_knapsack}, the canonical optimization problem for the quadratic knapsack problem is
\begin{equation}
    \max_{x \in \{0, 1\}^n} \quad f(x) = \sum_{i=1}^n p_i x_i + \sum_{i=1}^n \sum_{j > i}^n P_{ij} x_i x_j \quad \text{s.t.} \quad \sum_{i=1}^n w_i x_i \le C.
\end{equation}
The pairwise interaction terms introduce an energy landscape characterized by widespread multimodality and isolated modes, rendering uniform exploration and mode discovery non-trivial.
The reward is conventionally defined over terminating states as $R(x) = f(x)$ for feasible configurations ($R(x) = 0$ otherwise).

Item weights are drawn as $w_i \sim \mathcal{U}\{2, \dots, C-1\}$, linear profits as $p_i \sim \mathcal{U}(0.1, 1)$, and the interaction matrix as $P_{ij} = \sigma \varepsilon_{ij} b_{ij}$ for $j > i$ with $\varepsilon_{ij} \sim \mathcal{N}(0,1)$, $b_{ij} \sim \mathrm{Bernoulli}(\rho)$ at sparsity $\rho = 0.15$ and $\sigma = 1$, symmetrized as $P \leftarrow P + P^\top$.
Two adjustments make the landscape well-posed for a GFlowNet sampler. First, we add $\mathrm{diag}(\sum_j |P_{ij}| + 0.1)$ to the interaction matrix, which absorbs into the linear profit term for integral $x$ and guarantees $f(x) > 0$ on all reachable states, so that the log-reward $\log R(x) = \log f(x)$ is finite everywhere.
Second, interaction entries are rescaled by $1/\sqrt{w_i w_j}$, which decouples the magnitude of a pair's utility from the magnitude of its weight and prevents heavy items from dominating the profit surface purely by scale.
Trajectories insert one feasible item per step and terminate when the capacity admits no further insertion.
Feasibility is enforced by masking the forward policy's logits to items satisfying both $\sum_i w_i x_i + w_j \le C$ and $x_j < \kappa$, so the infeasible branch $R(x) = 0$ is never realized and no probability mass is wasted on rejection; the backward policy is uniform over removable items.
The instance ($w$, $p$, $P$) is drawn from a dedicated seed held fixed across all methods, distinct from the seed governing network initialization and sampling.

\paragraph{Comparisons and training protocol.} We compare four samplers, all sharing an identical masked-softmax MLP policy (2 hidden layers of width $d=512$, one scalar $\log Z$ per model) and optimized for 1000 iterations with Muon under a warmup--cosine schedule peaking at $10^{-3}$ for network parameters, and with Adam at $10^{-1}$ for $\log Z$.
\begin{itemize}[itemsep=0pt,topsep=0pt,left=8pt]
    \item TB trains a single GFlowNet with trajectory balance on $\varepsilon$-greedy on-policy rollouts ($\varepsilon = 0.1$), and serves as the no-diversity-mechanism reference.
    \item ACE \citep{dallantonia2026avoid}: the explorer minimizes a softplus-clamped TB residual against a tempered reward $R^\beta$ with $\beta = 0.25$, clamping only those trajectories whose student residual exceeds $\log \alpha$ with $\alpha = 0.3$, while the student combines TB on its own trajectories with TB on backward-evaluated explorer trajectories under the adaptive coefficient $\lambda = \sigma(\log Z - \log Z_{\mathrm{exp}})$. Rollouts are $\varepsilon$-greedy and on-policy ($\varepsilon = 0.1$).
    \item SA \citep{madan2025sibling} augments the sampler with a sibling network trained on a compounded objective that mixes an RND intrinsic bonus with the tempered extrinsic reward ($\beta = 0.25$), the main network taking equally weighted TB losses on its own and on the sibling's trajectories, also using $\varepsilon$-greedy on-policy rollouts ($\varepsilon = 0.1$).
    \item HCG trains an ensemble of two components strictly on-policy. The classifier is a two-hidden-layer MLP over raw terminal-state vectors, updated every iteration with $4$ gradient steps (Adam, $10^{-3}$) on class-balanced batches of 512 drawn from a per-component replay buffer of capacity $10^4$, under a $Z$-weighted with priority decay $a = 0.5$.
\end{itemize}

To equalize the per-iteration trajectory budget, methods that roll out two samplers per step (ACE, SA, HCG) use half the batch size of the single-rollout methods.

\subsection{Bit Sequences Experiment Details}
\label{app:sequences}

\paragraph{Environment Generation.} Following the variable definitions of Section \ref{sec:bit_sequences}, given a mode set $\mathcal{M}\subseteq\mathcal{X}$, we define
\[
    \log R(x)\;=\;\frac{1}{T}\Bigl(1-\frac{1}{M}\min_{m\in\mathcal{M}}d_H(x,m)\Bigr),\qquad T=\tfrac{1}{20},
\]
with $d_H$ the Hamming distance between two length-$M$ binary strings.
The log-reward is therefore maximal, equal to $20$ on the modes, and falls to $0$ at maximal distance from all of them, shrinking by a factor $e^{-20/M}$ for each additional mismatched bit.
Mass under $R$ is thus concentrated in narrow Hamming balls around $\mathcal{M}$, separated by broad regions of exponentially suppressed reward, so the task rewards a sampler that discovers and covers many well-separated basins rather than one that refines a single high-reward string.

The modes are compositional rather than spread uniformly over the hypercube.
Writing $K=8b$, each mode is formed by concatenating $b$ blocks drawn independently and uniformly from the five fixed 8-bit motifs $\texttt{00000000}$, $\texttt{11111111}$, $\texttt{11110000}$, $\texttt{00001111}$, and $\texttt{00111100}$, so that nearby modes share long substrings and the basins are structured rather than isotropic.
The mode set is drawn once under a fixed seed and held constant across methods and across learner seeds, so that all methods are compared against exactly the same target.

All methods use the same training configuration as in the Quadratic Knapsack experiment (Appendix~\ref{appendix:qknapsack}), also for $1000$ iterations.

\subsection{AMP Discovery Experiment Details}
\label{appendix:amp}

\paragraph{Environment.}
A trajectory writes a peptide of variable length $y_{1:L}$, $1\le L\le L_{\max}=10$, adding one amino acid per step from a 20-symbol vocabulary ($19$ amino acids, no Cysteine, plus \texttt{STOP}). When the sequence reaches length $L_{\max}$, termination is forced. The backward policy is the deterministic move that undoes the last forward action (removing the last token), so $\log P_B(\tau \mid y) = 0$ for every completed sequence.

\paragraph{Reward model.}
Following the model-based design protocol of \citet{angermueller2020population}, the reward is the maximum predicted activity $p_{\mathrm{RF}}(y)$ across an ensemble of five binary Random-Forest classifiers trained on GRAMPA, one per pathogen (\textit{E.\,coli}, \textit{S.\,aureus}, \textit{P.\,aeruginosa}, \textit{B.\,subtilis}, and \textit{C.\,albicans}). With cutoff $c=0.95$, temperature $T=0.3$, and margin $\Delta(y) = \operatorname{logit} p_{\mathrm{RF}}(y) - \operatorname{logit} c$,
\[
\log R(y) =
\begin{cases}
0, & p_{\mathrm{RF}}(y) \ge c,\\[2pt]
\dfrac{L+1}{T}\,\Delta(y), & p_{\mathrm{RF}}(y) < c,
\end{cases}
\qquad \text{clipped to } [-30, 0].
\]
Sequences above the cutoff saturate at unit reward, preventing a narrow set of top-scoring candidates from dominating the reward landscape \citep{dallantonia2026boostedgflownetsimprovingexploration}.

\paragraph{Training configuration.}
Table~\ref{tab:amp-defaults} lists the default hyperparameters and configurations shared across the variants reported on this task.

\begin{table}[h]
\centering
\caption{Default training setup for peptide generation.}
\label{tab:amp-defaults}
\begin{tabularx}{\textwidth}{l X}
\toprule
Component & Setting \\
\midrule
Sequences & 19 amino acids + \texttt{STOP}; $1\le L\le L_{\max}=10$ \\
Batch size & 4096 trajectories per rollout \\
Exploration $(\varepsilon)$ & 0.1 \\
Seeds & $\{10,\dots,14\}$ \\
Policy & Windowed autoregressive MLP (embedding 32, context 8, hidden 64) over $|\mathcal V|=20$ actions \\
Optimizer & AdamW; learning rates $(P_F, \log Z) = (5\times10^{-2}, 10^{-1})$ \\
Reward shaping & RF cutoff $c=0.95$; temperature $T=0.3$; $\log R$ clipped to $[-30,0]$ \\
\bottomrule
\end{tabularx}
\end{table}

\paragraph{Evaluation.}
Every $50$ epochs we draw $N=10000$ trajectories at $\varepsilon=0$, deduplicate terminal sequences, and score them with the RF ensemble. We report the cumulative count of distinct high-confidence sequences ($p_{\mathrm{RF}} \ge 0.95$) discovered across all evaluation checkpoints up to the current epoch, rather than the per-checkpoint count, so that the metric rewards sustained discovery of new effective peptides.

\section{Ablation: Classifier Update Frequency}
\label{appendix:ablation_classifier}
The ownership classifier $C_\phi$ is the only component shared by all GFlowNets.
Each component consults it on every update to decide which states lie in its disallowed region, and it is the sole channel through which the components influence one another, as everything else in the method is local.
The update frequency of $C_\phi$ therefore bounds how loosely the components can be coupled, and determines the communication cost when training is distributed.

This ablation measures this tolerance directly.
Let $c$ be the number of iterations between classifier updates, and $s$ the number of gradient steps taken at each update. Our default is $(c, s) = (1, 4)$.
We ask how far $c$ can be raised before HCG degrades, and whether any degradation is attributable to staleness or merely to performing fewer updates in total.

\paragraph{Setup.} We use the bit-sequence task at $M = 128$ with a binary alphabet and $1024$ modes.
The environment is generated once from a fixed seed and reused by every arm and every seed, so all runs face the same $1024$ modes.
We train two components for $2000$ iterations with a total batch of $64$ trajectories per iteration, split evenly across components so that every arm consumes the same number of trajectories.
HCG is trained strictly on-policy ($\epsilon = 0$).
The classifier uses a replay buffer of $10^4$ samples, balanced batches of $512$, and priority decay $a = 0.5$ in the $Z_i a_i$ weighting.

We sweep $c \in \{1, 10, 25, 50, 100\}$ in two families.
In the \emph{fixed-steps} family, $s = 4$ is held constant, so the total number of classifier gradient steps falls from $8000$ to $80$ as $c$ grows.
In the \emph{compute-matched} family, $s = 4c$, holding the total at $8000$ updates for every arm, so it varies staleness and burstiness at a fixed budget.
The two families meet at $(c, s) = (1, 4)$, the default configuration, which serves as their shared reference point.
We report a Trajectory Balance baseline at the same trajectory budget, trained off-policy with $\epsilon = 0.1$.
We use three seeds, which vary initialization and rollouts only, while the environment and reward instance are held fixed.

Pooled total variation is comparable only within a single evaluation call, because it is defined against a support pooled over the models being compared.
We therefore train every arm of a seed in one process and score them together in one call, over a shared support of $20480$ unique terminal states.
FCS is computed per model against uniform probes and is comparable across calls.
Top-200 mean reward is read from the training queue, which accumulates the highest-reward unique terminal states seen during training.

\begin{figure}[t]
    \centering
    \includegraphics[width=0.75\linewidth]{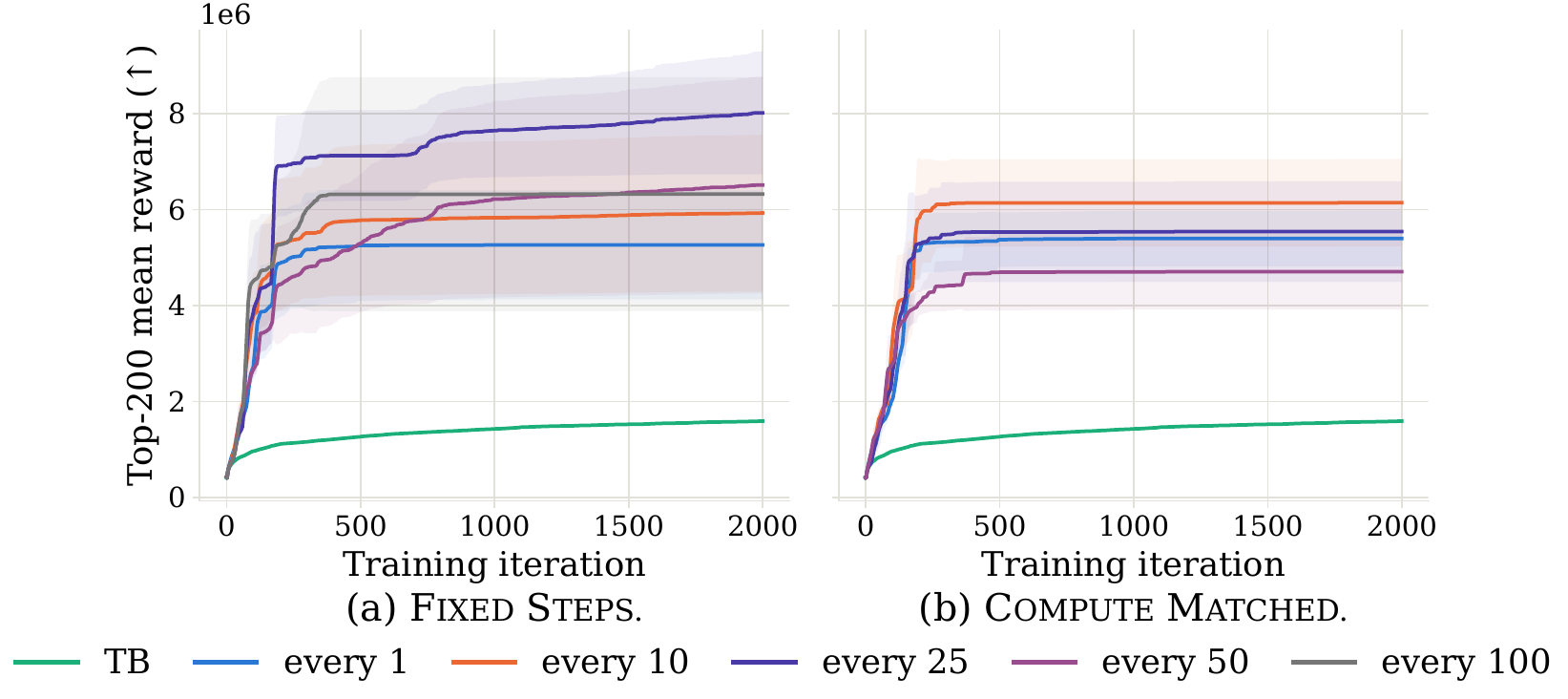}
    \caption{Effect of the classifier update cadence on the top-200 mean reward. Curves show the mean over three seeds, with bands at $\pm1$ standard deviation. Every cadence improves on TB and $c=25$ attains the highest mean in both families.}
    \label{fig:cadence-rewards}
\end{figure}

\begin{figure}[h]
    \centering
    \includegraphics[width=0.7\linewidth]{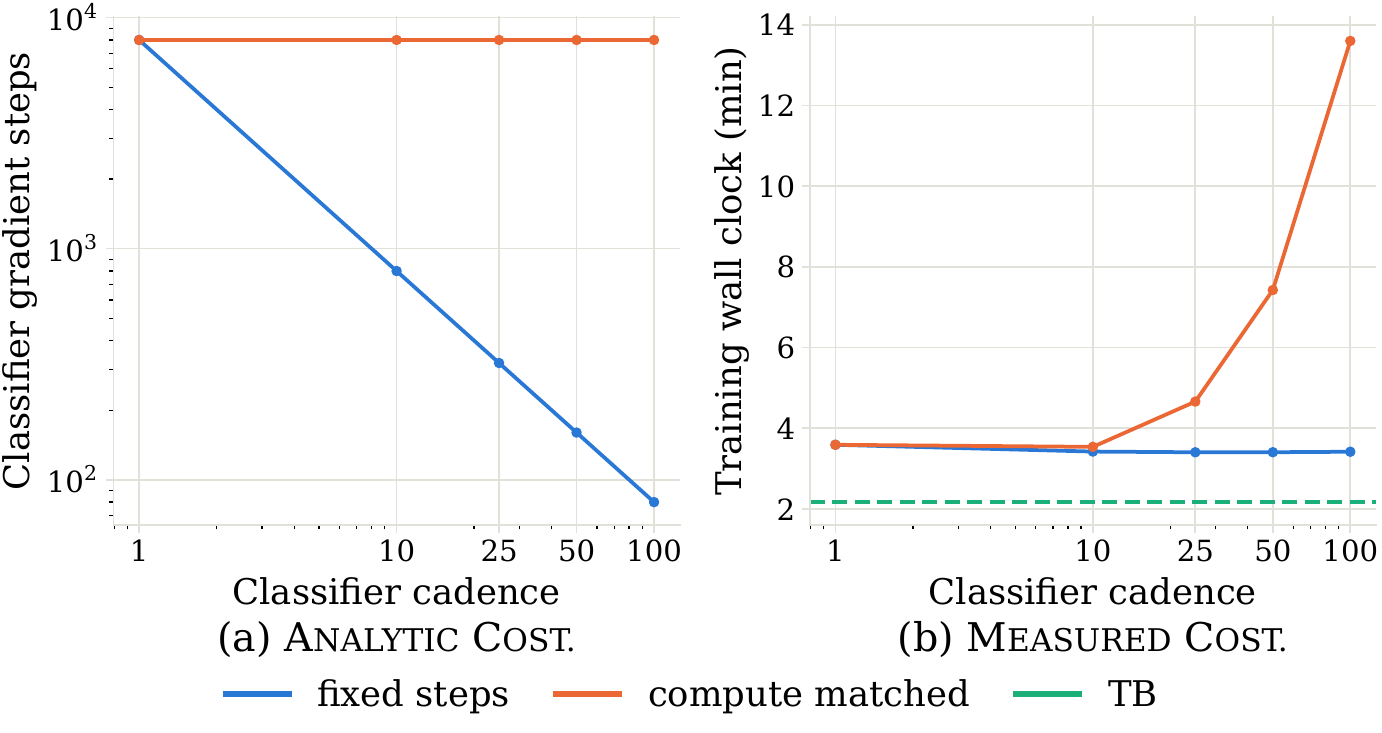}
    \caption{Cost of the classifier cadence. \textbf{(a)}~Classifier gradient steps: with fixed steps the number of updates falls in proportion to the cadence, from 8000 at $c=1$ to 80 at $c=100$, whereas the compute-matched family holds it at 8000. \textbf{(b)}~Measured training wall-clock time: sparser fixed-step cadences save almost no time despite performing fewer updates, while concentrating the same updates into rarer, longer bursts grows the cost. Trajectory Balance is shown for reference.}
    \label{fig:cadence-cost}
\end{figure}

\begin{table}[t]
  \centering\footnotesize
  \setlength{\tabcolsep}{5pt}
  \caption{\textbf{Fixed-steps family.} The classifier takes $s = 4$ gradient
  steps whenever it fires, so the total number of updates falls with the
  cadence $c$. Mean $\pm$ standard deviation over three seeds (42--44);
  $\Delta$ is the per-seed paired difference against the default $c = 1$.
  Reward is the mean of the top-200 unique terminal states. The TB reference
  is trained with $\epsilon = 0.1$, whereas HCG is on-policy
  ($\epsilon = 0$).}
  \label{tab:cadence-fixed-steps}
  \begin{tabular}{rrcccc}
    \toprule
     & Classifier & \multicolumn{2}{c}{Reward ($\times 10^{6}$, $\uparrow$)}
       & \multicolumn{2}{c}{Distance to target ($\downarrow$)} \\
    \cmidrule(lr){3-4} \cmidrule(lr){5-6}
    $c$ & updates & mean & $\Delta$ vs $c{=}1$ & Pooled TV & FCS \\
    \midrule
    TB  & 0    & $1.59 \pm 0.03$          & --               & $\mathbf{0.1677 \pm 0.0006}$ & $\mathbf{0.1657 \pm 0.0008}$ \\
    \midrule
    1   & 8000 & $5.26 \pm 1.14$          & --               & $0.1748 \pm 0.0044$ & $0.1717 \pm 0.0034$ \\
    10  & 800  & $5.93 \pm 1.64$          & $+0.67 \pm 0.54$ & $0.1731 \pm 0.0041$ & $0.1705 \pm 0.0041$ \\
    25  & 320  & $\mathbf{8.01 \pm 1.28}$ & $+2.75 \pm 1.61$ & $0.1715 \pm 0.0047$ & $0.1691 \pm 0.0045$ \\
    50  & 160  & $6.51 \pm 2.26$          & $+1.25 \pm 1.22$ & $0.1728 \pm 0.0046$ & $0.1702 \pm 0.0045$ \\
    100 & 80   & $6.32 \pm 2.44$          & $+1.06 \pm 1.36$ & $0.1726 \pm 0.0048$ & $0.1700 \pm 0.0047$ \\
    \bottomrule
  \end{tabular}
\end{table}

\begin{table}[t]
  \centering\footnotesize
  \setlength{\tabcolsep}{5pt}
  \caption{\textbf{Compute-matched family.} Steps per burst scale with the
  cadence, $s = 4c$, holding the total at 8000 updates for every arm, so this
  family varies staleness and burstiness at a fixed budget. Columns, seeds, and
  $\Delta$ as in Table~\ref{tab:cadence-fixed-steps}; $c = 1$ is the shared
  reference point of both families.}
  \label{tab:cadence-compute-matched}
  \begin{tabular}{rrcccc}
    \toprule
     & Classifier & \multicolumn{2}{c}{Reward ($\times 10^{6}$, $\uparrow$)}
       & \multicolumn{2}{c}{Distance to target ($\downarrow$)} \\
    \cmidrule(lr){3-4} \cmidrule(lr){5-6}
    $c$ & updates & mean & $\Delta$ vs $c{=}1$ & Pooled TV & FCS \\
    \midrule
    TB  & 0    & $1.59 \pm 0.03$          & --               & $\mathbf{0.1677 \pm 0.0006}$ & $\mathbf{0.1657 \pm 0.0008}$ \\
    \midrule
    1   & 8000 & $5.26 \pm 1.14$          & --               & $0.1748 \pm 0.0044$ & $0.1717 \pm 0.0034$ \\
    10  & 8000 & $5.40 \pm 0.58$          & $+0.13 \pm 1.35$ & $0.1778 \pm 0.0071$ & $0.1739 \pm 0.0058$ \\
    25  & 8000 & $\mathbf{6.14 \pm 0.91}$ & $+0.88 \pm 0.86$ & $0.1783 \pm 0.0051$ & $0.1737 \pm 0.0025$ \\
    50  & 8000 & $5.54 \pm 1.05$          & $+0.28 \pm 0.49$ & $0.1790 \pm 0.0061$ & $0.1735 \pm 0.0037$ \\
    100 & 8000 & $4.70 \pm 0.79$          & $-0.56 \pm 1.80$ & $0.1850 \pm 0.0239$ & $0.1735 \pm 0.0110$ \\
    \bottomrule
  \end{tabular}
\end{table}

\paragraph{Results} Table~\ref{tab:cadence-fixed-steps} shows no degradation anywhere in the fixed-steps family.
Against the default, every sparser arm improves FCS in all three seeds, and improves pooled TV and top-200 reward in two or three of them, so the paired deltas are favorable on all three metrics for all four values of $c$.
The largest effect is at $c = 25$, which improves reward by $2.75 \pm 1.61 \times 10^6$ with a negligible change in pooled TV, while performing $25\times$ fewer classifier updates.

The compute-matched family (Table~\ref{tab:cadence-compute-matched}) holds the total at $8000$ updates.
At $c = 100$, the two families hold staleness fixed and differ only in budget. The $80$ updates spread over training reach $6.32 \times 10^6$ reward and $0.1726$ pooled TV, while $8000$ updates delivered in bursts of $400$ reach $4.70 \times 10^6$ and $0.1850$.
A hundredfold larger classifier budget, at identical staleness, thus makes both metrics worse.
What harms the method is therefore not how old the classifier is, but how hard it is fitted to a replay window that does not refresh between bursts, which sharpens ownership onto stale evidence.

Figure~\ref{fig:cadence-rewards} reports the mean top-200 reward over training for all arms of the ablation.
In both the fixed-steps and the compute-matched families, the highest reward is obtained when the classifier is updated every 25 iterations.
Since the allowed and disallowed regions of each component change during training, the optimal update frequency may vary across environments. Rather than tuning this hyperparameter per environment, we use the default configuration $(c, s) = (1, 4)$ throughout, so the results reported in the main text may underestimate the performance achievable by HCG.

Figure~\ref{fig:cadence-cost} reports both the number of gradient steps and the wall-clock time spent in training.
By construction, the fixed-steps family falls from $8000$ gradient steps to $80$, while the compute-matched family is fixed at $8000$.
In wall-clock time the picture inverts, as the fixed-steps family remains flat at ${\approx}3.4$ minutes per arm, while the compute-matched family climbs to $13.6$ minutes.

\section{Wall-Clock-Matched Comparison}
\label{appendix:wall-clock}
The comparisons in the main text control for the number of training iterations, but methods differ in their cost per iteration. To check whether HCG's advantage holds under a fixed time budget, we plot the same runs against measured training time rather than iterations.

\begin{figure}[h]
    \centering
    \includegraphics[width=0.8\textwidth]{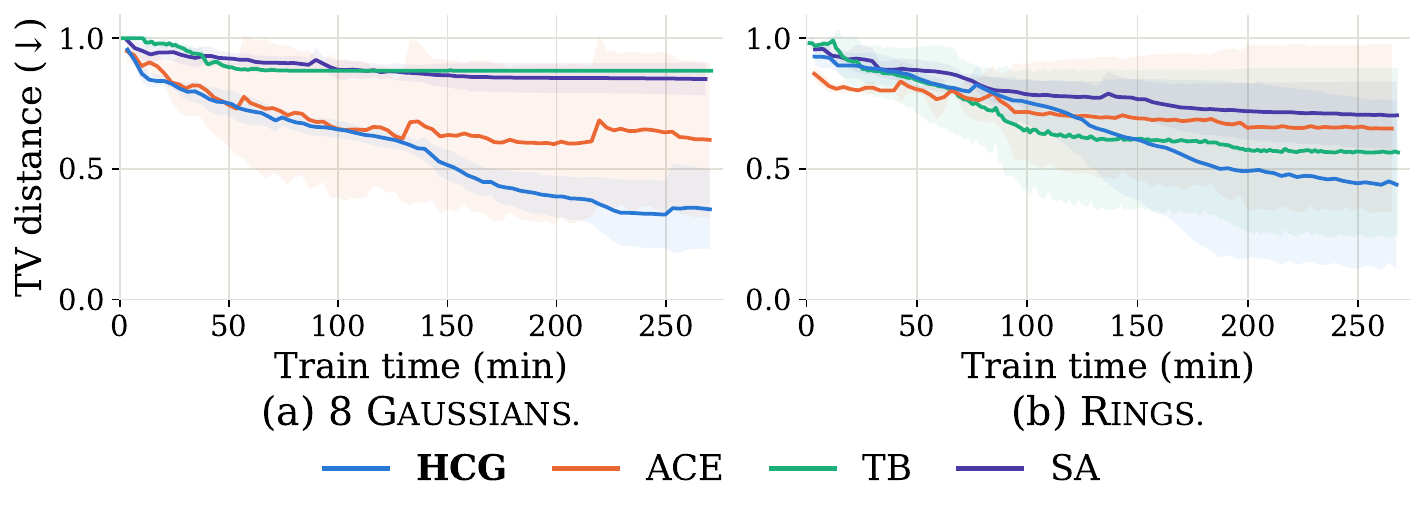}
    \caption{TV distance ($\downarrow$) to the target distribution against wall-clock
    training time on \textit{8 Gaussians} (left) and \textit{Rings}
    (right), truncated at TB's finishing time so that all methods are
    compared over the same time budget. Mean $\pm 1$ standard deviation over
    five seeds.}
    \label{fig:time-matched-coverage}
\end{figure}

\paragraph{Setup.} We use the DiffGrid setting of Section~\ref{ssec:diff_grid} ($81\times81$ grid, $T=80$) on the \textit{8 Gaussians} and \textit{Rings} landscapes, with batch size $1024$, exploration rate $\varepsilon=0.1$, five seeds per method, and HCG with 8 components trained in parallel. All runs were executed on the same CPU-only node with two 64-core CPUs (128 cores in total) and 512\,GB of RAM. Figure~\ref{fig:time-matched-coverage} plots the mean TV against cumulative training time, truncated when TB, the fastest method, completes its run, so that all curves share the same time budget.\looseness=-1

\paragraph{Results.} Under a matched time budget, HCG still reaches a lower TV distance than every baseline on both landscapes, indicating that its advantage is not an artifact of matching iterations. Although HCG performs more total computation per iteration, its components interact only through the shared classifier and can therefore be updated in parallel, so its wall-clock cost per iteration remains close to that of a single GFlowNet. ACE and SA, in contrast, are intrinsically sequential, as the main GFlowNet is trained on trajectories produced by the explorer, and the two models must be updated in turn, increasing their time per iteration. TB is the cheapest method per iteration and thus completes the most iterations within the budget, which explains why it is the closest baseline on \textit{Rings}. Even so, a single GFlowNet remains limited by the modes it discovers, and HCG outperforms it and the other baselines on both landscapes.

\section{HCG vs K-Independent GFNs}
\label{app:hcg-vs-independent-gfns}

\begin{figure}[t]
    \centering
    \includegraphics[width=0.5\linewidth]{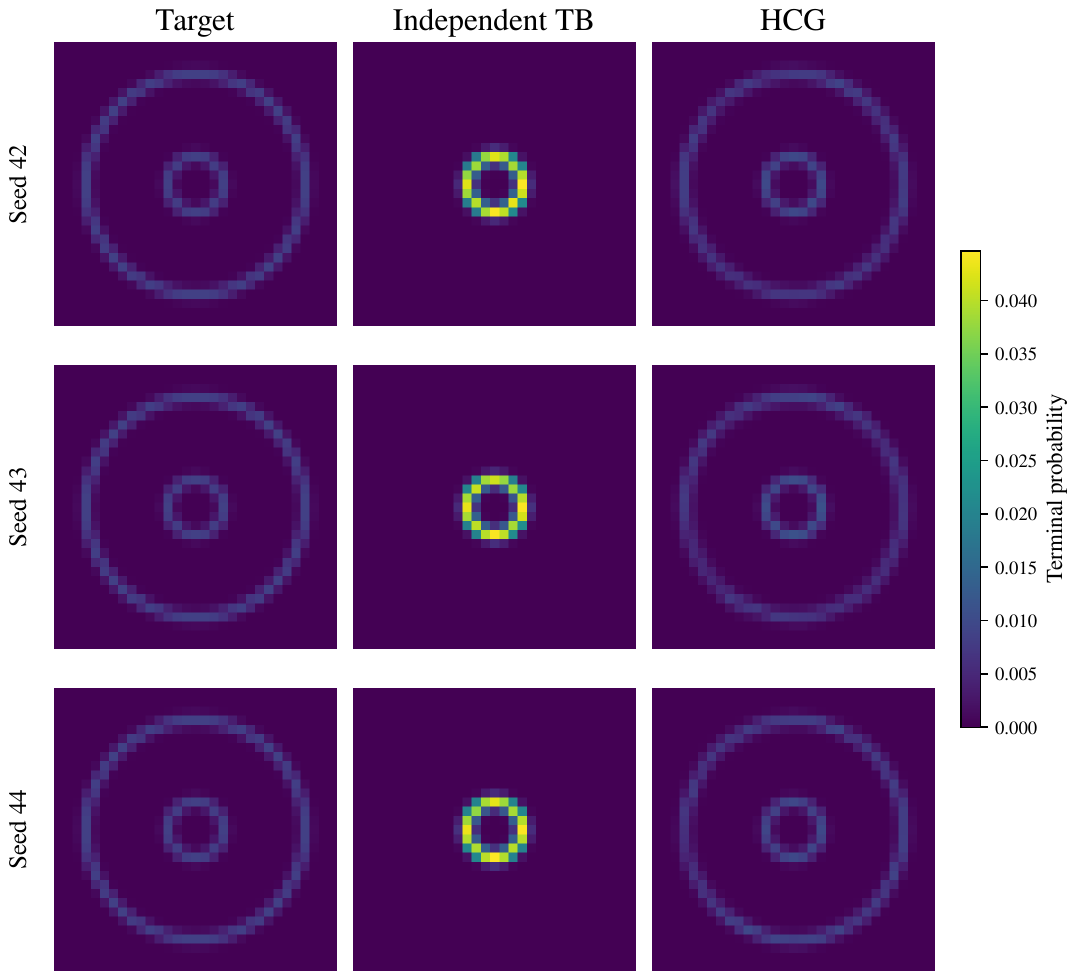}
    \caption{\textbf{HCG versus independent GFlowNet ensembles on Rings.}
    Exact terminal distributions on the $31\times31$ grid after 2,000 updates
    per GFlowNet. Rows correspond to seeds 42, 43, and 44, from top to bottom.
    Columns show the reward-proportional target, the mixture of three
    independently trained TB GFlowNets, and the three-component HCG mixture.
    Both learned mixtures use weights $Z_i/\sum_jZ_j$, and all nine panels
    share a linear probability color scale. Independent TB concentrates on
    the inner ring in all three runs, whereas HCG assigns substantial
    probability to both rings.}
    \label{fig:hcg-vs-independent-gfns}
\end{figure}

We investigate whether independently training multiple GFlowNets is sufficient
to improve mode coverage. We compare HCG with an ensemble of three independent
Trajectory Balance (TB) GFlowNets on the $31\times31$ Rings environment, using
three paired random seeds (42, 43, and 44). Both methods use three components
with identical GFlowNet architectures and paired initial parameters.
HCG additionally trains its ownership classifier; the comparison therefore
matches GFlowNet updates and trajectories, rather than total computational
cost. For both methods, we evaluate the final mixture with learned weights
$\omega_i=Z_i/\sum_j Z_j$. Terminal probabilities are computed exactly by
forward dynamic programming, without evaluation sampling.

Figure~\ref{fig:hcg-vs-independent-gfns} shows that, in every seed, all three
independent TB policies concentrate on the same easy inner ring, leaving
negligible probability on the distant ring. Their mixtures have total
variation distance $0.7996$ from the target in all three runs, whereas HCG
achieves $0.0816\pm0.0243$ (mean $\pm$ sample standard deviation). Defining
the outer region as $A=\{(x,y):\sqrt{x^2+y^2}>7.5\}$, its target probability
is $0.7996$. HCG assigns $0.7700\pm0.0146$ probability to this region,
compared with less than $1.1\times10^{-10}$ under each independent-TB mixture.

The failure of the independent ensemble cannot be repaired by reweighting
these trained components. Every independent policy satisfies
$P_i(A)<2.4\times10^{-10}$. Hence, for any convex mixture with weights
$\alpha_i\geq0$ and $\sum_i\alpha_i=1$,
\begin{equation}
    P_{\mathrm{mix}}(A)
    =\sum_i\alpha_iP_i(A)
    \leq\max_iP_i(A)
    <2.4\times10^{-10}.
\end{equation}
Thus, independent replication preserves the same exploration failure in
this finite-budget experiment, while HCG substantially improves coverage
of the distant ring and the fit to the reward-proportional target.

\end{document}